%% file: perturbed_arxiv.tex
\documentclass[11pt]{article}
\usepackage{setspace}
\usepackage{amsfonts,bm,xspace,amsthm,fullpage,amssymb,dsfont}
\usepackage{multirow, algorithm,rotating}
\usepackage{graphicx}

\usepackage{url,cite}
\usepackage{enumerate}
\usepackage{xcolor,soul}
\usepackage{thm-restate}
\usepackage[small,bf]{caption}

\usepackage{standalone}
\usepackage{verbatim}
\usepackage{setspace}
\usepackage{paralist}
\usepackage{mdwmath}
\usepackage{amssymb}
\usepackage{amsmath}
\usepackage{amsthm}
\usepackage{xfrac}
\usepackage{mathtools}  
\usepackage{chngcntr}
\usepackage{tabulary}
\usepackage{booktabs}
\usepackage{epsfig}
\usepackage{wrapfig}
\usepackage{lipsum}
\usepackage{caption}
\usepackage{subcaption}
\usepackage{multirow}
\usepackage{algcompatible}
\usepackage{footnote}
\usepackage{paralist}
\algblockdefx{PARFOR}{ENDPARFOR}[1]%
  {\textbf{for all }#1 \textbf{do in parallel}}%
  {\textbf{end for}}

\algblockdefx{PWHILE}{ENDPWHILE}[1]%
  {\textbf{while }#1 \textbf{do in parallel}}%
  {\textbf{end while}}

\newtheorem{theorem}{Theorem}
\newtheorem{definition}[theorem]{Definition}

\newtheorem{remark}{Remark}
\newtheorem{lemma}[theorem]{Lemma}

\newtheorem{corollary}[theorem]{Corollary}

\newtheorem{asm}{Assumption}

\newcommand{\indi}{\mathds{1}}

\newcommand{\RR}{\mathbb{R}}

\newcommand{\EE}{\mathbb{E}}

\newcommand{\diag}{\operatorname{diag}}

\newcommand{\vecx}{\mathbf{x}}
\newcommand{\vecv}{\mathbf{v}}

\newcommand{\vecy}{\mathbf{y}}
\newcommand{\vecg}{\mathbf{g}}

\newcommand{\vecn}{\mathbf{n}}
\newcommand{\Exp}{\mathbb{E}}

\newcommand{\matP}{\mathbf{P}}
\newcommand{\matD}{\mathbf{D}}
\newcommand{\HW}{\textsc{Hogwild!}}
\newcommand{\setS}{\mathcal{S}}
\newcommand{\matS}{{\bf S}}

\newcommand{\CM}{\textsc{KroMagnon}}
\newcommand{\cdeg}{\overline{\Delta}_\text{C}}
\newcommand{\ct}{\mathcal{O}(1)}

\numberwithin{equation}{section}

\begin{document}

\title{\huge Perturbed Iterate Analysis \\for Asynchronous Stochastic Optimization}

\author{
Horia Mania$^{\alpha,\epsilon}$, Xinghao Pan$^{\alpha, \epsilon}$, Dimitris Papailiopoulos$^{\alpha,\epsilon}$\\
Benjamin Recht$^{\alpha,\epsilon, \sigma}$, Kannan Ramchandran$^{\epsilon}$, and Michael I. Jordan$^{\alpha, \epsilon, \sigma}$\\
{\normalsize$^{\alpha}$AMPLab, $^{\epsilon}$EECS at UC Berkeley, $^{\sigma}$Statistics at UC Berkeley}
} 

\date{}

\maketitle

\begin{abstract}
We introduce and analyze stochastic optimization methods where the input to each update is perturbed by bounded noise.  
We show that this framework forms the basis of a unified approach to analyze asynchronous implementations of stochastic optimization algorithms, 
by viewing them as  serial methods operating on noisy inputs.  
Using our perturbed iterate framework, we provide new analyses of the \HW{} algorithm and asynchronous stochastic coordinate descent, that are simpler than earlier analyses, remove many assumptions of previous models, and in some cases yield improved upper bounds on the convergence rates.  
We proceed to apply our framework to develop and analyze \CM{}: a novel, parallel, sparse stochastic variance-reduced gradient (SVRG) algorithm.  
We demonstrate experimentally on a 16-core machine that the sparse and parallel version of SVRG is in some cases more than four orders of magnitude faster than the standard SVRG algorithm.
\\

\end{abstract}

\input{intro.tex}
\input{noisy_grad.tex} 
\input{hogwild.tex}

\input{ascd.tex}
\input{svrg.tex}
\input{experiments}

\section{Conclusions and Open Problems}

We have introduced a novel framework for analyzing parallel asynchronous stochastic gradient optimization algorithms. The main advantage of our framework is that it is straightforward to apply to a range of first-order stochastic algorithms, while it involves elementary derivations. Moreover, in our analysis we lift, or relax, many of the assumptions made in prior art, \emph{e.g.}, we do not assume consistent reads, and we analyze full stochastic gradient updates.  We use our framework to analyze \HW{} and ASCD, and further introduce and analyze \CM{}, a new asynchronous sparse SVRG algorithm. 

We conclude with some open problems:
\begin{enumerate}
\item It would be interesting to obtain tighter bounds for the convergence of function values of the algorithms presented. 
How do  the ``errors" due to asynchrony influence the convergence rate of function values? 
In this case the number of iterations required to reach a target accuracy should scale with the condition number of the objective, not its square. Moreover, the literature on stochastic coordinate descent establishes convergence results in terms of coordinate-wise Lipschitz constants---a more refined smoothness quantity than the full-function smoothness. It would be worthwhile to know if our framework can be adapted to take these parameters into account.

\item Our perturbed iterates framework relies fundamentally on the strong convexity assumption.
However, asynchronous algorithms are known to perform well on non-strongly convex (and even nonconvex) objectives.
Can we generalize our framework to simply convex, or smooth functions? Under what assumptions, or simple families of functions, can we show convergence for nonconvex problems?

\item As previously explained, we believe that the upper bounds on $\tau$---the proxy for the number of cores---in our ASCD and \CM{} analyses  are amenable to improvements. 
It is an open problem to explore the extent of such improvements.

\item  Our analysis offers sensible upper bounds only in the presence of sparsity. It seems, however, that to obtain speedup results for \HW{}, it is only necessary to have small correlation between randomly sampled gradients. In what practical setups do randomly selected gradients have sufficiently small correlation? 
Does that immediately imply linear speedups in the same way that update sparsity does?

\item In this work we analyzed three similar stochastic first-order methods. It is an open problem to apply our framework and provide an elementary analysis for a greater variety of stochastic gradient type optimization algorithms, such as AdaGrad-type schemes (similar to \cite{duchi2013estimation}), or stochastic dual coordinate methods (similar to \cite{hsieh2015passcode}).

\item  Capturing the effects of asynchrony as noise on the algorithmic input seems to be applicable to settings beyond stochastic optimization. 
As shown recently for a combinatorial graph problem, a similar viewpoint enables the analysis of an asynchronous graph clustering algorithm \cite{2015arXiv150705086P}. It is an interesting endeavor to explore the extent to which a perturbed iterate viewpoint is suitable for analyzing general asynchronous iterative algorithms. 
\end{enumerate}

\subsection*{Acknowledgments}
This work was supported in part by the Mathematical Data Science program of the
Office of Naval Research. BR is generously supported by ONR awards , N00014-15-1-2620, N00014-13-1-0129,  and N00014-14-1-0024 and NSF awards CCF-1148243 and CCF-1217058. This research is supported in part by the  NSF CISE Expeditions Award 7076018 and gifts from Amazon Web Services, Google, IBM, SAP, The Thomas and Stacey Siebel Foundation, Adatao, Adobe, Apple Inc., Blue Goji, Bosch, Cisco, Cray, Cloudera, Ericsson, Facebook, Fujitsu, Guavus, HP, Huawei, Intel, Microsoft, Pivotal, Samsung, Schlumberger, Splunk, State Farm, Virdata and VMware.

\begingroup
\emergencystretch 1.5em
\sloppy
\bibliographystyle{unsrt}
\bibliography{perturbed_iterates}
\endgroup
\appendix
\input{proofs}

\end{document}

%% file: intro.tex
\section{Introduction}

Asynchronous parallel stochastic optimization algorithms have recently gained significant traction in algorithmic machine learning.  A large body of recent work has demonstrated that near-linear speedups are achievable, in theory and practice, on many common machine learning tasks
\cite{niu2011hogwild,
recht2012factoring,
zhuang2013fast,
yun2013nomad,
liu2014asynchronous1,
duchi2013estimation,
wang2014asynchronous,
hsieh2015passcode}.  Moreover, when these lock-free algorithms are applied to non-convex optimization, significant speedups are still achieved with no loss of statistical accuracy.  This behavior has been demonstrated in practice in 
state-of-the-art deep learning systems such as Google's Downpour SGD~\cite{dean2012large}  and Microsoft's Project Adam~\cite{chilimbi2014project}.

Although asynchronous stochastic algorithms are simple to implement and enjoy excellent performance in practice, they are challenging to analyze theoretically. 
The current analyses require lengthy derivations and several assumptions that may not reflect realistic system behavior.
Moreover, due to the difficult nature of the proofs, the algorithms analyzed are often simplified versions of those actually run in practice.

In this paper, we propose a general framework for deriving convergence rates for parallel, lock-free, asynchronous first-order stochastic algorithms. 
We interpret the algorithmic effects of asynchrony as perturbing the stochastic iterates with bounded noise.
This interpretation allows us to show how a variety of asynchronous first-order algorithms can be analyzed as their serial counterparts operating on noisy inputs. 
The advantage of our framework is that it yields elementary convergence proofs, can remove or relax simplifying assumptions adopted in prior art,  and can yield improved bounds when compared to earlier work.

We demonstrate the general applicability of our framework by providing new convergence analyses for \HW{}, \emph{i.e.}, the asynchronous stochastic gradient method (SGM), for asynchronous stochastic coordinate descent (ASCD), and \CM{}: a novel asynchronous sparse version of the stochastic variance-reduced gradient (SVRG) method~ \cite{johnson2013accelerating}. In particular, we provide a modified version of SVRG that allows for sparse updates, we show that this method can be parallelized in the asynchronous model,  and we provide convergence guarantees using our framework.  Experimentally, the asynchronous, parallel sparse SVRG achieves nearly-linear speedups on a machine with 16 cores and is sometimes four orders of magnitude faster than the standard (dense) SVRG method.

 \subsection{Related work}
 
 The algorithmic tapestry of parallel stochastic optimization is rich and diverse extending back at least to the late 60s
\cite{chazan1969chaotic}.
Much of the contemporary work in this space is built upon the foundational work of Bertsekas, Tsitsiklis et al.
\cite{bertsekas1989parallel,
tsitsiklis1986distributed}; the shared memory access model that we are using in this work, is very similar to the partially asynchronous model introduced in the aforementioned manuscripts.
Recent advances in parallel and distributed computing technologies have generated renewed interest in the theoretical understanding and practical implementation of parallel stochastic algorithms
\cite{zinkevich2009slow,
zinkevich2010parallelized,
gemulla2011large,
agarwal2011distributed,
richtarik2012parallel,
jaggi2014communication}.

 The power of lock-free, asynchronous stochastic optimization on shared-memory multicore systems was first demonstrated  in the work of \cite{niu2011hogwild}.
The authors introduce \HW{}, a completely lock-free and asynchronous parallel stochastic gradient method (SGM) that exhibits nearly linear speedups for a variety of machine learning tasks.
Inspired by \HW{}, several authors developed lock-free and asynchronous algorithms that move beyond SGM, such as the work of~Liu et al. on parallel stochastic coordinate descent~ \cite{liu2014asynchronous1,liu2015asynchronous}.  Additional work in first order optimization and beyond
\cite{duchi2013estimation,
wang2014asynchronous,
hong2014distributed,
hsieh2015passcode,
feyzmahdavian2015asynchronous},
extending to parallel iterative linear solvers \cite{liu2014asynchronous2,avron2014revisiting}, has further shown that linear speedups are possible in the asynchronous shared memory model.



%% file: noisy_grad.tex
\section{Perturbed Stochastic Gradients}
\label{section: noisy input stochastic gradients}
\paragraph{Preliminaries and Notation}
We study parallel asynchronous iterative algorithms that minimize convex functions $f(\vecx)$ with $\vecx\in\RR^d$. 
The computational model is the same as that of Niu et al.~\cite{niu2011hogwild}: a number of cores have access to the same shared memory, and each of them can read and update components of $\vecx$ in the shared memory.
The algorithms that we consider are asynchronous and lock-free: cores do not coordinate their reads or writes, and while a core is reading/writing other cores can update the shared variables in $\vecx$. 

We focus our analysis on functions $f$ that are $L$-smooth and $m$-strongly convex. 
A function $f$ is $L$-smooth if it is differentiable and has Lipschitz gradients
\begin{equation}
\nonumber
\|\nabla f(\vecx) - \nabla f(\vecy)\| \leq L \|\vecx - \vecy\| \text{ for all } \vecx,\vecy\in \RR^d,
\end{equation}
where $\|\cdot\|$ denotes the Euclidean norm. 
Strong convexity with parameter $m > 0$ imposes a curvature condition on $f$:
\begin{equation}
\nonumber
f(\vecx) \geq f(\vecy) + \langle\nabla f(\vecy), \vecx - \vecy\rangle + \frac{m}{2}\|\vecx - \vecy\|^2 \text{ for all } \vecx, \vecy \in \RR^d.
\end{equation}
Strong convexity implies that $f$ has a unique minimum $\vecx^*$ 
and satisfies
\begin{equation}
\langle\nabla f(\vecx) - \nabla f(\vecy), \vecx -\vecy\rangle\geq m\|\vecx - \vecy\|^2.
\nonumber
\end{equation}
In the following, we use $i$, $j$, and $k$ to denote iteration counters, while reserving $v$ and $u$ to denote coordinate indices. We use $\ct$ to denote absolute constants. 

\paragraph{Perturbed Iterates}
A popular way to minimize convex functions is via {\it first-order stochastic} algorithms.
These algorithms can be described using the following general iterative expression:
\begin{equation}
\vecx_{j+1} = \vecx_j - \gamma \vecg(\vecx_j, \xi_j),
\label{eq:sgm_iterate}
\end{equation}
where $\xi_j$ is a random variable independent of $\vecx_j$ and $\vecg$ is an unbiased estimator of the true gradient of $f$ at $\vecx_j$: $\EE_{\xi_j} \vecg (\vecx_j, \xi_j) = \nabla f(\vecx_j)$.
The success of first-order stochastic techniques partly lies in their computational efficiency: the small computational cost of using noisy gradient estimates trumps the gains of using true gradients. 

A major advantage of the iterative formula in \eqref{eq:sgm_iterate} is that---in combination with strong convexity, and smoothness inequalities---one can easily track algorithmic progress and establish convergence rates to the optimal solution.
Unfortunately, the progress of asynchronous parallel algorithms cannot be precisely described or analyzed using the above iterative framework. 
Processors do not read from memory actual iterates $\vecx_j$, as there is no global clock that synchronizes reads or writes while different cores write/read ``stale" variables.

In the subsequent sections, we show that the following simple perturbed variant of Eq.~\eqref{eq:sgm_iterate} can capture the algorithmic progress of asynchronous stochastic algorithms. Consider the following iteration
\begin{equation}
\label{eq: general algorithm}
\vecx_{j+1} = \vecx_j - \gamma \vecg(\vecx_j + \vecn_j, \xi_j),
\end{equation}
where $\vecn_j$ is a stochastic error term. For simplicity let $\hat\vecx_j = \vecx_j+\vecn_j$.
Then,
\begin{align}
&\|\vecx_{j+1} - \vecx^*\|^2 
= \|\vecx_j - \gamma \vecg(\hat\vecx_j,\xi_j) - \vecx^*\|^2  \nonumber\\
&= \|\vecx_j - \vecx^*\|^2- 2\gamma
\langle \vecx_j - \vecx^*, \vecg(\hat\vecx_j,\xi_j)\rangle + \gamma^2 \|\vecg(\hat\vecx_j,\xi_j)\|^2\label{eq:aj_first}\\
&= \|\vecx_j - \vecx^*\|^2- 2\gamma
\langle \hat\vecx_j - \vecx^*, \vecg(\hat\vecx_j,\xi_j)\rangle + \gamma^2 \|\vecg(\hat\vecx_j,\xi_j)\|^2
+ 2\gamma\langle \hat\vecx_j - \vecx_j, \vecg(\hat\vecx_j,\xi_j)\rangle, \nonumber
\end{align}
where in the last equation we added and subtracted the term $2\gamma\langle \hat\vecx_j , \vecg(\hat\vecx_j,\xi_j)\rangle$. 

We assume that $\hat\vecx_j$ and $\xi_j$ are independent.
However, in contrast to recursion \eqref{eq:sgm_iterate}, we no longer require $\vecx_j$ to be independent of $\xi_j$.
The importance of the above independence assumption will become clear in the next section. 

We now take the expectation of both sides in \eqref{eq:aj_first}. 
Since $\hat\vecx_j$ and $\vecx^*$ are independent of $\xi_j$, we use iterated expectations to obtain
$\EE \langle \hat\vecx_j -\vecx^*, \vecg(\hat\vecx_j,\xi_j)\rangle = \EE \langle \hat\vecx_j - \vecx^*, \nabla f(\hat\vecx_j)\rangle.$
Moreover, since $f$ is $m$-strongly convex, we know that
\begin{equation}
\label{equation: strong conv lower bound}
\langle \hat\vecx_j - \vecx^*, \nabla f(\hat\vecx_j)\rangle \geq m \|\hat\vecx_j - \vecx^*\|^2 \geq \frac{m}{2}\|\vecx_j  - \vecx^*\|^2 - m\|\hat\vecx_j - \vecx_j\|^2,
\end{equation}
where the second inequality is a simple consequence of the triangle inequality. 
Now, let $a_j = \EE\|\vecx_j - \vecx^*\|^2$ and substitute \eqref{equation: strong conv lower bound} back into Eq.~\eqref{eq:aj_first} to get
{\small
\begin{align}
a_{j+1} &\leq (1-\gamma m) a_j  +  \gamma^2\underbrace{\EE\|g(\hat\vecx_j,\xi_j)\|^2}_{R_0^j} + 2\gamma m \underbrace{\Exp\|\hat\vecx_j - \vecx_j\|^2}_{R_1^j} +  2\gamma \underbrace{\EE \langle \hat\vecx_j - \vecx_j, \vecg(\hat\vecx_j,\xi_j)\rangle}_{R_2^j}.
\label{eq:aj_main}
\end{align}}The recursive equation \eqref{eq:aj_main} is key to our analysis. We show that for given $R_0^j$, $R_1^j$, and $R_2^j$, we can obtain convergence rates through elementary algebraic manipulations.
Observe that there are three ``error" terms in \eqref{eq:aj_main}: $R_0^j$ captures the stochastic gradient decay with each iteration, $R_1^j$ captures the mismatch  between the true iterate and its noisy estimate, and $R_2^j$ measures the size of the projection of that mismatch on the gradient at each step.
The key contribution of our work is to show that 1) this iteration can capture the algorithmic progress of asynchronous algorithms, and 2) the error terms can be bounded to obtain a $\mathcal{O}(\log(1/\epsilon)/\epsilon)$ rate for \HW{}, and linear rates of convergence for asynchronous SCD and asynchronous sparse SVRG.


%% file: hogwild.tex
\section{Analyzing \HW{}}
\label{section: hogwild}

In this section, we provide a simple analysis of \HW{}, the asynchronous implementation of SGM.
We focus on functions $f$ that are decomposable into $n$ terms:
\begin{equation}
\label{eq:decomposable-f}
f(\vecx) = \frac{1}{n}\sum_{i=1}^n f_{e_i}(\vecx),
\end{equation}
where $\vecx \in \RR^d$, and each $f_{e_i}(\vecx)$  depends only on the coordinates indexed by the subset $e_i$ of $\{1,2,\ldots,d\}$.
For simplicity we assume that the terms of $f$ are differentiable; our results can be readily extended to non-differentiable $f_{e_i}$s.

We refer to the sets $e_i$ as \emph{hyperedges} and denote the set of hyperedges by $\mathcal{E}$.
We sometimes refer to $f_{e_i}$s as the \emph{terms} of $f$.
As shown in Fig.~\ref{fig:bipartite_graph},
the hyperedges induce a bipartite graph between the $n$ terms and the $d$ variables in $\vecx$, and a conflict graph between the $n$ terms. 
Let  $\cdeg$ be the average degree in the conflict graph; that is, the average number of terms that are in conflict with a single term.
We assume that $\cdeg\ge 1$, otherwise we could decompose the problem into smaller independent sub-problems.
As we will see,
under our perturbed iterate analysis framework the convergence rate of asynchronous algorithms depends on $\cdeg$.

\begin{figure}[t]
\centerline{\includegraphics[width=0.62\columnwidth]{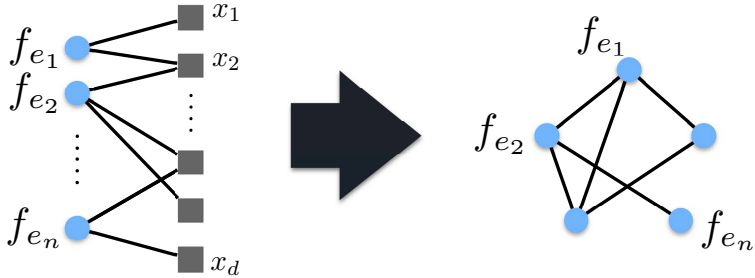}}
\caption{\small
The bipartite graph on the left has as its leftmost vertices the $n$ function terms and as its rightmost vertices the coordinates of $\vecx$.
A term $f_{e_i}$ is connected to a coordinate $x_j$ if hyperedge $e_i$ contains $j$ (\emph{i.e.}, if the $i$-th term is a function of that coordinate).
The graph on the right depicts a conflict graph between the function terms.
The vertices denote the function terms, and two terms are joined by an edge if they conflict on at least one coordinate in the bipartite graph.
}
\label{fig:bipartite_graph}  
\hrule
\end{figure}

\HW{} (Alg.~\ref{algo:HW}) is a method to parallelize SGM in the asynchronous setting~\cite{niu2011hogwild}.
It is deployed on multiple cores that have access to shared memory, where the optimization variable $\vecx$ and the data points that define the $f$ terms are stored.
During its execution
each core samples uniformly at random a hyperedge $s$ from $\mathcal{E}$. It reads the coordinates $v\in s$ of the shared vector $\vecx$, evaluates $\nabla f_s$ at the point read, and finally adds $- \gamma \nabla f_s$ to the shared variable. 

\begin{center}
\begin{algorithm}[h]
   \caption{\HW{}}
\begin{algorithmic}[1]
\PWHILE{ number of sampled hyperedges $\le T$}
\STATE sample a random hyperedge $s$
\STATE $[\hat\vecx]_s$ = an inconsistent read of the shared variable $[\vecx]_s$ 
\STATE $[{\bf u}]_s = -\gamma\cdot \vecg([\hat\vecx]_s, s) $ 
\FOR{$v\in s$}
\STATE $[\vecx]_v = [\vecx]_v + [{\bf u}]_v$ \hfill{\color{gray}// atomic write}
\ENDFOR
\ENDPWHILE
\end{algorithmic}
   \label{algo:HW}
 \end{algorithm}
\end{center}

During the execution of \HW{} cores do not synchronize or follow an order between reads or writes.
Moreover, they access ({\it i.e.}, read or write) a set of coordinates in $\vecx$ without the use of any locking mechanisms that would ensure a conflict-free execution.
This implies that the reads/writes of distinct cores can intertwine in arbitrary ways, \emph{e.g.}, while a core updates a subset of variables, before completing its task, other cores can read/write the same subset of variables.

{\color{black}
In \cite{niu2011hogwild}, the authors analyzed a variant of \HW{} in which several simplifying
assumptions were made.  Specifically, in \cite{niu2011hogwild}
{\it 1)}  only a single coordinate per sampled hyperedge is updated (\emph{i.e.}, the for loop in \HW{} is replaced with a single coordinate update);
{\it 2)} the authors assumed \emph{consistent reads}, {\it i.e.}, it was assumed that  while a core is reading the shared variable, no writes from other cores occur;
{\it 3)} the authors make an implicit assumption on the uniformity of the processing times of cores (explained in the following), that does not generically hold in practice.
These simplifications alleviate some of the challenges in analyzing \HW{} and allowed the authors to provide a convergence result.  As we show in the current paper, however, these simplifications are not necessary to obtain a convergence analysis.  Our perturbed iterates framework can be used in an elementary way to analyze the original version of \HW{}, yielding improved bounds compared to earlier analyses.
}

\subsection{Ordering the samples}

{\color{black}
A subtle but important point in the analysis of \HW{} is the need to define an order for the sampled hyperedges.  
A key point of difference of our work is that \emph{we order the samples based on the order in which they were sampled}, not the order in which cores complete the processing of the samples. 

\begin{definition}
We denote by $s_i$ the $i$-th sampled hyperedge in a run of Alg.~\ref{algo:HW}.
\end{definition}
That is, $s_i$ denotes the sample obtained when line $2$ in Alg.~\ref{algo:HW} is executed for the $i$-th time.  This is different from the original work of \cite{niu2011hogwild}, in which the samples were ordered according to the completion time of each thread.  
The issue with such an ordering is that the distribution of the samples, conditioned on the ordering, is not always uniform; for example, hyperedges of small cardinality are more likely to be ``early" samples.  
A uniform distribution is needed for the theoretical analysis of stochastic gradient methods, a point that is disregarded in \cite{niu2011hogwild}.  Our ordering according to sampling time resolves this issue by guaranteeing uniformity among samples in a trivial way.

\subsection{Defining read iterates and clarifying independence assumptions}
Since the shared memory variable can change inconsistently during reads and writes, we also have to be careful about the notion of {\it iterates} in \HW{}.
\begin{definition}
\label{def:iterates}
We denote by $\overline \vecx_i$ the contents of the shared memory before the $i$-th execution of line $2$.
%
Moreover, we denote by $\hat{\vecx}_i\in \RR^d$ the vector, that in coordinates  $v\in s_i$ contains exactly what the core that sampled $s_i$ read.
We then define $[\hat\vecx_i]_v = [\overline \vecx_i]_{v}$ for all  $v\not\in s_i$.
Note that we do not assume {\it consistent reads}, \emph{i.e.}, the contents of the shared memory can potentially change while a core is reading. 
\end{definition}

At this point we would like to briefly discuss an {\it independence} assumption held by all prior work.
In the following paragraph, we explain why this assumption is not always true in practice.
In Appendix~\ref{sec:lift_independence}, we show how to lift such independence assumption, but for ease of exposition we do adopt it in our main text.

\begin{asm} 
The vector $\hat{\vecx}_i$ is independent of the sampled hyperedge $s_i$.
\label{asm:ind}
\end{asm}

The above independence assumption is important when establishing the convergence rate of the algorithm, and has been held explicitly or implicitly in prior work \cite{niu2011hogwild,duchi2013estimation,liu2014asynchronous1,liu2015asynchronous}. 
Specifically, when proving convergence rates for these algorithms we need to show via iterated expectations that
$\EE\left<\hat{\bf x}_i-{\bf x}^*, g(\hat{\bf x}_i, s_i)\right> = \left<\hat{\bf x}_i-{\bf x}^*, \nabla(\hat{\bf x}_i)\right>$, which follows from the independence of $\hat{\bf x}_i$ and  $s_i$. 
However,  observe that although $\overline \vecx_i$ is independent of $s_i$  by construction, this is not the case for the vector $\hat{\vecx}_i$ read by the core that sampled $s_i$.
For example, consider the scenario of two consecutively sampled  hyperedges in Alg.~\ref{algo:HW} that overlap on a subset of coordinates.
Then, say one core is reading the coordinates of the shared variables indexed by its hyperedge, while the second core is updating a subset of these coordinates. 
In this case, the values read by the first core depend on the support of the sampled hyperedge.

One way to rigorously enforce the independence of $\hat\vecx_i$ and $s_i$ is to require the processors to read the {\it entire} shared variable $\vecx$ before sampling a new hyperedge. 
However, this might not be reasonable in practice, as the dimension of $\vecx$ tends to be considerably larger than the sparsity of the hyperedges.
As we mentioned earlier, in Appendix~\ref{sec:lift_independence}, we show how to overcome the issue of dependence and thereby remove Assumption~\ref{asm:ind}; 
however, this results in a slightly more cumbersome analysis.
To ease readability, in our main text we do adopt Assumption~\ref{asm:ind}.


}

\subsection{The perturbed iterates view of asynchrony}
In this work, we assume that all writes are atomic, in the sense that they will be successfully recorded in the shared memory at some point.
Atomicity is a reasonable assumption in practice, as it can be strictly enforced through compare-and-swap operations~\cite{niu2011hogwild}.

\begin{asm}
\label{asm:writes}
Every write in line 6 of Alg.~\ref{algo:HW} will complete successfully.
\end{asm}

This assumption implies that all writes will appear in the shared memory by the end of the execution, in the form of coordinate-wise updates. 
Due to commutativity the order in which these updates are recorded in the shared memory is irrelevant.
Hence, after processing a total of $T$ hyperedges the shared memory contains:
\begin{equation}
\underbrace{\overbrace{\vecx_0 - \gamma \vecg(\hat\vecx_0, s_0) }^{\vecx_1} - \ldots - \gamma \vecg(\hat\vecx_{T - 1},s_{T-1})}_{\vecx_{T}},
\label{eq:RAMcontents}
\end{equation}
where $\vecx_0$ is the initial guess and {\color{black}${\bf x}_i$ is defined as the vector that contains {\it all} gradient updates up to sample $s_{i-1}$.}

\begin{remark}
Throughout this section we denote $\vecg(\vecx, s_j) = \nabla f_{s_j}(\vecx)$, which we assume to be bounded: $\|\vecg(\vecx, s) \| \leq M$.
Such a uniform bound on the norm of the stochastic gradient is true when operating on a bounded $\ell_\infty$ ball; this can in turn be enforced by a simple, coordinate-wise thresholding operator.
We can refine our analysis by avoiding the uniform bound on $\|\vecg(\vecx,s)\|$, through a simple application of the co-coercivity lemma as it was used in \cite{needell2014stochastic}; 
in this case, our derivations would only require a uniform bound on $\|\vecg(\vecx^*,s)\|$. 
Our subsequent derivations can be adapted to the above, however to keep our derivations elementary we will use the uniform bound on $\|\vecg(\vecx, s) \|$.
\end{remark}

\begin{remark}
Observe that although a core is only reading the subset of variables that are indexed by its sampled hyperedge, in \eqref{eq:RAMcontents} we use the entire vector $\hat {\bf x}$ as the input to the sampled gradient.
We can do this since $\vecg(\hat\vecx_{k},s_{k})$ is independent of the coordinates of $\hat\vecx_{k}$ outside the support of hyperedge $s_k$.
\end{remark}

Using the above definitions, we define the perturbed iterates of \HW{} as 
\begin{equation}
\vecx_{i+1} = \vecx_i - \gamma \vecg(\hat\vecx_i, s_i),
\label{eq:hogwild-updates}
\end{equation}
for $i = 0,1,..., T-1$, where $s_i$ is the $i$-th uniformly sampled hyperedge.
Observe that all but the first and last of these iterates are ``fake": 
there might not be an actual time when they exist in the shared memory during the execution.
However, $\vecx_{0}$ is what is stored in memory before the execution starts, and $\vecx_{T}$ is exactly what is stored in shared memory at the end of the execution. 

We observe that the iterates in \eqref{eq:hogwild-updates} place \HW{} in the perturbed gradient framework introduced in \S \ref{section: noisy input stochastic gradients}:
$$
a_{j+1} 
\leq (1-\gamma m) a_j  + \gamma^2\underbrace{\EE\|g(\hat\vecx_j,s_j)\|^2}_{R_0^j} + 2\gamma m \underbrace{\Exp\|\hat\vecx_j - \vecx_j\|^2}_{R_1^j} +  2\gamma \underbrace{\EE \langle \hat\vecx_j - \vecx_j, \vecg(\hat\vecx_j,s_j)\rangle}_{R_2^j}.$$
We are only left to bound the three error terms $R_0^j$, $R_1^j$, and $R_2^j$.
{\color{black}
Before we proceed, we note that for the technical soundness of our theorems, we have to also define a random variable that captures the {\it system randomness}.
 In particular, let $\xi$ denote a random variable that encodes the randomness of the system ({\it i.e.}, random delays between reads and writes, gradient computation time, etc). 
Although we do not  explicitly use $\xi$, its distribution is required implicitly to compute the expectations for the convergence analysis.
This is because the random samples $s_0, s_1, \ldots, s_{T-1}$ do not fully determine the output of Alg.~\ref{algo:HW}.
However, $s_0, \ldots,s_{T-1}$ along with $\xi$ completely determine the time of all reads and writes.
We continue with our final assumption needed by our analysis, that is also needed by prior art.}

\begin{asm} [Bounded overlaps]
Two hyperedges $s_i$ and $s_j$ overlap in time if they are processed concurrently at some point during the execution of \HW{}. 
The time during which a hyperedge $s_i$ is being processed begins when the sampling function is called and ends after the last coordinate of $\vecg(\hat\vecx_{i},s_i)$ is written to the shared memory.
We assume that there exists a number $\tau \geq 0$, such that the maximum number of sampled hyperedges that can overlap in time with a particular sampled hyperedge cannot be more than $\tau$. 
\label{asm:tau}
\end{asm}

The usefulness of the above assumption is that it essentially abstracts away all system details relative to delays, processing overlaps, and number of cores into a single parameter.
Intuitively, $\tau$ can be perceived as a proxy for the number of cores, {\it i.e.}, we would expect that no more than roughly $O(\#\text{cores})$ sampled hyperedges overlap in time with a single hyperedge, assuming that the processing times across the samples are approximately similar.
Observe that if $\tau$ is small, then we expect the distance between $\vecx_j$ and the noisy iterate $\hat\vecx_j$ to be small.
In our perturbed iterate framework, if we set $\tau=0$, then we obtain the classical iterative formula of serial SGM.

To quantify the distance between $\hat\vecx_j$ ({\it i.e.}, the iterate read by the core that sampled $s_j$) and $\vecx_j$ ({\it i.e.}, the ``fake" iterate used to establish convergence rates), 
we observe that any difference between them is caused solely by hyperedges that overlap with $s_j$ in time. To see this, let $s_i$ be an ``earlier" sample, \emph{i.e.}, $i< j$, that does not overlap with $s_j$ in time.
This implies that the processing of $s_i$ finishes before $s_j$ starts being processed.
Hence, the {\it full} contribution of $\gamma \vecg(\hat\vecx_i,s_i)$ will be recorded in both $\hat\vecx_j$ and $\vecx_j$ (for the latter this holds by definition). Similarly, if $i>j$ and $s_i$ does not overlap with $s_j$ in time, then neither $\hat\vecx_j$ nor $\vecx_j$ (for the latter, again by definition) contain {\it any} of the coordinate updates involved in the gradient update $\gamma\vecg(\hat\vecx_i,s_i)$. Assumption~\ref{asm:tau} ensures that if $i < j-\tau$ or $i> j+\tau$, the sample $s_i$ {\it does not} overlap in time with $s_j$. 

By the above discussion, and due to Assumption~\ref{asm:tau}, there exist diagonal matrices $\matS_{i}^j$ with diagonal entries in $\{-1,0,1\}$ such that 
\begin{equation}
\hat\vecx_j - \vecx_j = \sum_{\substack{i = j - \tau,\;i\neq j}}^{j + \tau}\gamma\matS_i^j\vecg(\hat\vecx_i , s_i).  
\end{equation}
These diagonal matrices account for any possible pattern of (potentially) partial updates that can occur while hyperedge $s_j$ is being processed.
We would like to note that the above notation bears resemblance to the coordinate-update mismatch formulation of asynchronous coordinate-based algorithms, as in \cite{liu2015asynchronous,lian2015asynchronous, peng2015arock}. 

We now turn to the convergence proof, emphasizing its elementary nature within the perturbed iterate analysis framework.  We begin by bounding the error terms $R_1^j$ and $R_2^j$ ($R_0^j$ is already assumed to be at most $M^2$).

\begin{lemma}
\label{lemma: bound extra noise}
\HW{} satisfies the recursion in \eqref{eq:aj_main} with  
{\small
\begin{align}
R_1^j = \EE \|\hat\vecx_j-\vecx_j\|^2 \le \gamma^2 M^2 \left(2\tau + 8\tau^2 \frac{\cdeg}{n}\right)\nonumber \text{ and }
R_2^j = \EE \langle \hat\vecx_j-\vecx_j, \vecg(\hat\vecx_j, s_j)\rangle \leq 4\gamma M^2\tau \frac{\cdeg}{n},\nonumber
\end{align}
}where $\cdeg$ is the average degree of the conflict graph between the hyperedges. 
\end{lemma}
\begin{proof}
The norm of the mismatch can be bounded in the following way:
{\small
\begin{align*}
R_1^j &= \gamma^2\EE \biggl\|\sum_{\substack{i = j - \tau\\ i\neq j}}^{j + \tau} \matS_i^j \vecg(\hat\vecx_i, s_i)\biggr\|^2 \hspace{-0.1cm} \leq  \gamma^2 \sum_{i}\EE \|\matS_i^j\vecg(\hat\vecx_i, s_i)\|^2 + \gamma^2\sum_{\substack{i,k\\i\neq k}} \EE \left|\langle \matS_i^j\vecg(\hat\vecx_i, s_i), \matS_k^j\vecg(\hat\vecx_k, s_k)\rangle\right| \\
&\leq  \gamma^2 \sum_{i}\EE \|\vecg(\hat\vecx_i, s_i)\|^2 + \gamma^2\sum_{\substack{i,k\\i\neq k}} \EE \left\{\|\vecg(\hat\vecx_i, s_i)\|\|\vecg(\hat\vecx_k, s_k)\|\indi(s_i\cap s_k \neq \emptyset)\right\},
\end{align*}
}since $\matS_i^j$ are diagonal sign matrices and since the steps $\vecg(\hat\vecx_i,s_i)$ are supported on the samples $s_i$. We use the upper bound $\|\vecg(\hat\vecx_i, s_i)\| \leq M$ to obtain
$$
R_1^j \leq 2\tau\cdot \gamma^2 M^2 + \gamma^2 M^2(2\tau)^2\Pr (s_i \cap s_j \neq \emptyset) \leq \gamma^2 M^2 \cdot \left(2\tau + 8\tau^2 \frac{\cdeg}{n}\right). 
$$
The last step  follows because two sampled hyperedges (sampled with replacement) intersect with probability  at most $2\frac{\cdeg}{n}$. 
We can bound $R_2^j$ in a similar way:

\begin{align*}
 R_2^j &= \gamma \sum_{\substack{i = j - \tau\\ i\neq j}}^{j + \tau} \EE\langle{\bf S}_i^j\vecg(\hat\vecx_i,s_i), \vecg(\hat\vecx_j,s_j)\rangle
 \le \gamma M^2 \sum_{\substack{i = j - \tau\\ i\neq j}}^{j + \tau} \EE \left\{\indi(s_i \cap s_j \neq \emptyset)\right\}
 \leq 4 \gamma  M^2 \tau \frac{\cdeg}{n}.
\end{align*}

\end{proof}

Plugging  the bounds of Lemma~\ref{lemma: bound extra noise} in our recursive formula, we see that \HW{} satisfies the recursion
\begin{align}
a_{j+1} &= \left( 1- \gamma  m\right)a_j + \gamma^2 M^2\biggl(1 + \underbrace{8\tau \frac{\cdeg}{n}+ 4\gamma m \tau + 16\gamma m \tau^2 \frac{\cdeg}{n}}_{\delta}\biggr).
\label{eq:HWrecursion}
\end{align} 
On the other hand, serial SGM satisfies the recursion 
$a_{j+1} \leq \left( 1- \gamma  m\right)a_j + \gamma^2 M^2.$
If the step size is set to $\gamma = \frac{\epsilon m}{2M^2}$, it attains target accuracy $\epsilon$ in
$T \geq 2 M^2/(\epsilon m^2) \log\left(\frac{2a_0}{\epsilon}\right)$ iterations.
{\color{black}Hence, when the term $\delta$  of \eqref{eq:HWrecursion}
 is order-wise constant, 
\HW{} satisfies the same recursion (up to constants) as serial SGM.
This directly implies the main result of this section.}

\begin{theorem}
\label{thm:HW}
If the number of samples that overlap in time with a single sample during the execution of \HW{} is bounded as 
$$\tau =\mathcal{O}\left(\min\left\{\frac{n}{\cdeg}, \frac{M^2}{\epsilon m^2} \right\}\right),$$ 
\HW{}, with step size $\gamma = \frac{\epsilon m}{2M^2}$, reaches an accuracy of $\EE\|\vecx_k-\vecx^*\|^2\le \epsilon$ after
$$
T \geq \ct \frac{M^2 \log\left(\frac{a_0}{\epsilon}\right)}{\epsilon m^2}
$$
iterations.
\end{theorem}

{\color{black}Since the iteration bound in the theorem is (up to a constant) the same as that of serial SGM, our result implies a linear speedup.}
We would like to note that an improved rate of $O(1/\epsilon)$ can be obtained by appropriately diminishing stepsizes per epoch (see, \emph{e.g.}, \cite{niu2011hogwild,hazan2014beyond}).  
Furthermore, observe that although the $\frac{M^2}{\epsilon m^2}$ bound on $\tau$ might seem restrictive, it is---up to a logarithmic factor---proportional to the total number of iterations required by \HW{} (or even serial SGM) to reach $\epsilon$ accuracy.
 Assuming that the average degree of the conflict graph is constant, and that we perform a constant number of passes over the data, \emph{i.e.}, $ T = c\cdot n$, then
 $\tau$ can be as large as $\tilde{\mathcal{O}}(n)$,  \emph{i.e.}, nearly linear in the number of function terms.\footnote{$\tilde{\mathcal{O}}$ hides logarithmic terms.}

\subsection{Comparison with the original \HW{} analysis of \cite{niu2011hogwild}}

Let us summarize the key points of improvement compared to the original \HW{} analysis:
\begin{itemize}
\item Our analysis is elementary and compact, and follows simply by  bounding the $R_0^j,R_1^j$, and $R_2^j$ terms, after introducing the perturbed gradient framework of \S~\ref{section: noisy input stochastic gradients}.
\item We do not assume consistent reads: while a core is reading from the shared memory other cores are allowed to read, or write.
\item In \cite{niu2011hogwild} the authors analyze a simplified version of \HW{} where for each sampled hyperedge only a randomly selected coordinate is updated. Here we analyze the ``full-update" version of \HW{}.
\item We order the samples by the order in which they were sampled, not by completion time. This allows to rigorously prove our convergence bounds, without assuming anything on the distribution of the processing time of each hyperedge.
This is unlike \cite{niu2011hogwild}, where there is an implicit assumption of uniformity with respect to processing times.
\item The previous work of \cite{niu2011hogwild} establishes a nearly-linear speedup for \HW{} if $\tau$ is bounded as 
$\tau = \mathcal{O}\left(\sqrt[4]{\sfrac{n}{\Delta_\text{R} \Delta_\text{L}^2}}\right)$,
where $\Delta_{\text{R}}$ is the maximum right degree of the term-variables bipartite graph, shown in Fig~\ref{fig:bipartite_graph}, and $\Delta_{\text{L}}$ is the maximum left degree of the same graph. Observe that $\Delta_{\text{R}}\cdot \Delta_{\text{L}}^2 \ge \Delta_{\text{L}}\cdot \Delta_{\text{C}}$, where $\Delta_{\text{C}}$ is the maximum degree of the conflict graph.
Here, we obtain a linear speedup for up to 
$\tau =\mathcal{O}\left(\min\left\{\sfrac{n}{\cdeg}, \sfrac{M^2}{\epsilon m^2} \right\}\right)$, 
where $\cdeg$ is only the average degree of the conflict graph in Fig~\ref{fig:bipartite_graph}. 
Our bound on the delays can be orders of magnitude better than that of \cite{niu2011hogwild}.
\end{itemize}


%% file: ascd.tex
\section{Asynchronous Stochastic Coordinate Descent}

In this section, we use the perturbed gradient framework to analyze the convergence 
of asynchronous parallel stochastic coordinate descent (ASCD). This algorithm has been previously analyzed in \cite{liu2014asynchronous1,liu2015asynchronous}.  We show that the algorithm admits an elementary treatment in our perturbed iterate framework, under the same assumptions made for \HW{}.

\begin{center}
\begin{algorithm}[H]
   \caption{ASCD}        
\begin{algorithmic}[1]
\PWHILE{iterations $\le T$}
\STATE $\hat\vecx = $ an inconsistent read of the shared variable $\vecx$
\STATE Sample a coordinate $s$
\STATE $u_s= -\gamma\cdot d[\nabla f(\vecx)]_s$ 
\STATE $[\vecx]_s = [\vecx]_s + u_s$ \hfill{\color{gray}// atomic write}
\ENDPWHILE
\end{algorithmic}
   \label{algo:ASCD}
 \end{algorithm}
\end{center}

ASCD, shown in Alg.~\ref{algo:ASCD}, is a linearly convergent algorithm for minimizing strongly convex functions $f$.
At each iteration a core samples one of the coordinates, computes a full gradient update for that coordinate, and proceeds with updating a single element of the shared memory variable $\vecx$.
The challenge in analyzing ASCD, compared to \HW{}, is that, in order to show linear convergence, we need to show that the error due to the asynchrony between cores decays fast when the iterates arrive close to the optimal solution. 
The perturbed iterate framework can handle this type of noise analysis in a straightforward manner, using simple recursive bounds.

We define $\hat\vecx_i$ as in the previous section, but now the samples $s_i$ are coordinates sampled uniformly at random from $\{1,2,\ldots, d\}$. After $T$ samples have been processed completely, the following vector is contained in  shared memory:
\begin{equation}
\underbrace{\overbrace{\vecx_0 - \gamma d [\nabla f(\hat\vecx_0)]_{s_0}{\bf e}_{s_0} }^{\vecx_1} - \ldots - \gamma  d[\nabla f(\hat\vecx_{T-1})]_{s_{T-1}}{\bf e}_{s_{T-1}}}_{\vecx_{T}},
\nonumber
\end{equation}
where $\vecx_0$ is the initial guess, ${\bf e}_{s_j}$ is the standard basis vector with a one at position $s_j$, $[\nabla f(\vecx)]_{s_j}$ denotes the $s_j$-th coordinate of the gradient of $f$ computed at $\vecx$. Similar to \HW{} in the previous section, ASCD satisfies the following iterative formula
$$
\vecx_{j+1} = \vecx_j - \gamma \cdot  d\cdot [\nabla f(\hat\vecx_j)]_{s_j}{\bf e}_{s_j} = \vecx_j - \gamma\cdot \vecg(\hat\vecx_j, s_j).
$$
Notice that $\EE_{s_j} \vecg(\hat\vecx_j, s_j) = \nabla f(\hat\vecx_j)$, and thus, similarly to \HW{}, ASCD's  iterates $a_j=\EE\|\vecx_j-\vecx^*\|^2$ satisfy the recursion of Eq.~\eqref{eq:aj_main}:
\begin{equation}
a_{j+1} 
\leq \left(1- \gamma m\right) a_j  + \gamma^2\underbrace{\EE\|g(\hat\vecx_j,s_j)\|^2}_{R_0^j} + 2\gamma m \underbrace{\Exp\|\hat\vecx_j - \vecx_j\|^2}_{R_1^j} +  2\gamma \underbrace{\EE \langle \hat\vecx_j - \vecx_j, \vecg(\hat\vecx_j,s_j)\rangle}_{R_2^j}.
\nonumber
\end{equation}	

Before stating the main result of this section, let us introduce some further notation.
Let us define the largest distance between the optimal vector, and the vector read by the cores during the execution of the algorithm:
$
\hat a_0 \coloneqq \max_{0\le k\le T}\EE\|\hat\vecx_k-\vecx^*\|^2,
$
a value which should be thought of as proportional to $a_0 = \EE\|\vecx_0-\vecx^*\|^2$.
Furthermore, by a simple application of the $L$-Lipschitz assumption on $f$, we have a uniform bound on the norm of each computed gradient
$
M^2\coloneqq\max_{0\le k\le T}\EE\|\nabla f(\hat\vecx_k)\|^2 \le L^2\hat a_0. 
$
Here we assume that the optimization takes place in an $\ell_\infty$ ball, so that $M < \infty$. This simply means that the iterates will never have infinitely large coordinate values. This assumption is made in previous work explicitly or implicitly, and in practice it can be implemented easily since the projection on an $\ell_\infty$ ball can be done component-wise. 
Finally, let us define the condition number of $f$ as 
$\kappa \coloneqq \sfrac{L}{m},$ 
where $L$ is the Lipschitz constant, and $m$ the strong convexity parameter.
\begin{theorem}
Let the maximum number of coordinates that can be concurrently processed while a core is processing a single coordinate be at most
$$\tau = \mathcal{O}\left(\min\left\{\frac{\kappa\sqrt{d}}{\log\left(\frac{\hat a_0}{\epsilon}\right)}, \sqrt[6]{d}\right\}\right).$$
Then, ASCD with step-size  $\gamma =  \frac{\ct}{dL\kappa}$ achieves an error $\EE\|\vecx_k-\vecx^*\|^2\le \epsilon$ after 
$$k \ge \ct \cdot d \kappa^2\log\left(\frac{a_0}{\epsilon}\right)$$
iterations.
\label{theo:ASCD}
\end{theorem}

Using the recursive inequality \eqref{eq:aj_main}, serial SCD with the same step-size as in the above theorem, can be shown to achieve the same accuracy as ASCD in the same number of steps. 
Hence, as long as the proxy for the number of cores is bounded as 
$\tau = \mathcal{O}(\min\{\kappa\sqrt{d}\log(\sfrac{\hat a_0}{\epsilon})^{-1}, \sqrt[6]{d}\})$,
our theorem implies a linear speedup with respect to this simple convergence bound.
We would like to note, however, that the coordinate descent literature sometimes uses more refined properties of the function to be optimized that can lead to potentially better convergence bounds, especially in terms of function value accuracy, {\it i.e.}, $f(\vecx_k)-f(\vecx^*)$ (see {\it e.g.}, \cite{liu2014asynchronous1,liu2015asynchronous,bubeck2014theory}). 

We would further like to remark that between the two bounds on $\tau$, the second one, {\it i.e.}, $\mathcal{O}(\sqrt[6]{d})$, is the more restrictive, as the first one is proportional---up to log factors---to the square root of the number of iterations, which is usually $\Omega(d)$. We explain in our subsequent derivation how this loose bound can be improved, but leave the tightness of the bound as an open question for future work. 

\subsection{Proof of Theorem~\ref{theo:ASCD}}

The analysis here is slightly more involved compared to \HW{}.
The main technical bottleneck is to relate the decay of $R_0^j$ with that of $R_1^j$, and then to exploit the sparsity of the updates for bounding $R_2^j$.

We start with a simple upper bound on the norm of the gradient updates.
From the $L$-Lipschitz assumption on $\nabla f(\vecx)$, we have
$$\EE_{s_k}\left\|\vecg(\hat\vecx_{k},s_k)\right\|^2= d\cdot \|\nabla f(\hat{\vecx}_k)\|^2 \leq d L^2 \|\hat{\vecx}_k - \vecx^*\|^2 \leq 2dL^2\|\vecx_j - \vecx^*\|^2 + 2dL^2\|\vecx_j - \hat{\vecx}_k\|^2,$$
where the last inequality is due to Jensen's inequality.
This yields the following result.

\begin{lemma}
\label{lemma: shifted grad bound ascd} 
For any $k$ and $j$ we have
$\EE\left\|\vecg(\hat\vecx_{k},s_k)\right\|^2 \leq 2dL^2\left( a_j +  \EE\|\vecx_j -\hat{\vecx}_k\|^2\right).$
\end{lemma}
Let $T$ be the total number of ASCD iterations, and let us define the set
$$ \setS_r^j = \{\max\{j-r\tau,0\},\ldots,j -1, j, j +1, \ldots, \min\{j+r\tau,T\}\},$$
which has cardinality at most $2r\tau+1$ and contains all indices around $j$ within $r\tau$ steps, as sketched in Fig.~\ref{fig:timeline}.
\begin{figure}[t]
\centerline{\includegraphics[width=0.5\columnwidth]{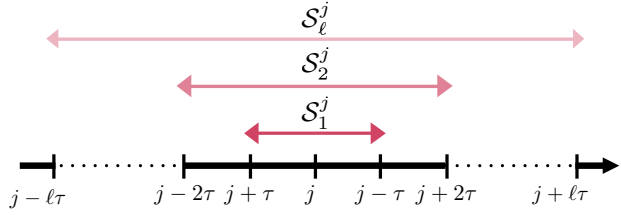}}
\caption{The set
$ \setS_r^j = \{\max\{j-r\tau,0\},\ldots,j -1, j, j +1, \ldots, \min\{j+r\tau,T\}\}$
comprises the indices around $j$ (including $j$) within $r\tau$ steps.
The cardinality of such a set is $2r\tau+1$. Here, $\setS_0^j =\{j\} $.
}
\label{fig:timeline}
\hrule
\end{figure}
Due to Assumption \ref{asm:tau}, and similar to \cite{liu2015asynchronous}, there exist variables $\sigma_{i,k}^j \in \{-1,0,1\}$ such that, for any index $k$ in the set $\setS_r^j$, we have
\begin{equation}
\label{eq: shift decomposition ascd}
\hat\vecx_k - \vecx_j = \sum_{i \in \mathcal{S}^j_{r+1}}\sigma_{i,k}^j \gamma\vecg(\hat\vecx_i, s_i). 
\end{equation}
The above equation implies that the difference between a ``fake" iterate at time $j$ and the value that was read at time $k$ can be expressed as a linear combination of any coordinate updates that occurred during the time interval defined by 
$\mathcal{S}^j_{r+1}$.

From Eq. \eqref{eq: shift decomposition ascd} we see that $\|\hat\vecx_k - \vecx_j\|$, for any $k\in\setS_r^j$, can be upper bounded in terms of the magnitude of the coordinate updates that occur in $\setS_{r+1}^j$. 
Since these updates are coordinates of the true gradient, we can use their norm to bound the size of $\hat\vecx_k - \vecx_j$. This will be useful towards bounding $R_1^j$.
Moreover, Lemma~\ref{lemma: shifted grad bound ascd} shows that the magnitude of the gradient steps can be upper bounded in terms of the size of the mismatches. This will in turn be useful in bounding $R_0^j$.
The above observations are fundamental to our approach. 
The following lemma makes the above ideas explicit.

\begin{lemma}
For any $j\in\{0,\ldots, T\}$, we have
\begin{align}
\max_{k\in\setS^j_r}\EE\|\vecg(\hat\vecx_{k},s_k)\|^2 &\le 2dL^2\left( a_j +  \max_{k \in \setS^j_r} \EE \|\hat\vecx_k - \vecx_j\|^2 \right)\\
\max_{k \in \setS^j_r} \EE \|\hat\vecx_k - \vecx_j\|^2 &\le (3\gamma \tau (r+1))^2 \max_{k\in\setS^j_{r+1}}\EE\|\vecg(\hat\vecx_{k},s_k)\|.
\end{align}
\label{lemma:recursions ascd}
\end{lemma}
\begin{proof}
The first inequality is a consequence of Lemma~\ref{lemma: shifted grad bound ascd}.
For the second, as mentioned previously, we have $\hat{\vecx}_k - \vecx_j = \sum_{i\in \setS_{r+1}^j}  \sigma_{i,k}\gamma\vecg(\hat{\vecx}_i, s_i)$ when $k\in \setS_r^j$.
Hence,
\begin{align*}
\EE \|\hat{\vecx}_k - \vecx_j\|^2
& = \gamma^2\cdot\EE \biggl\{\biggl\|\sum_{i\in \setS_{r+1}^j} \sigma_{i,k}^j \cdot \vecg(\hat{\vecx}_i, s_i)\biggr\|^2\biggr\}
 \le \gamma^2\cdot\EE \biggl\{|\setS_{r+1}^j| \sum_{i\in \setS_{r+1}^j} \|\vecg(\hat{\vecx}_i, s_i)\|^2\biggr\}\\
&\le \gamma^2\cdot|\setS_{r+1}^j|^2 \max_{i\in\setS_{r+1}^j} \EE \|\vecg(\hat{\vecx}_i, s_i)\|^2
\le (3\gamma\tau(r+1))^2\max_{i\in\setS_{r+1}^j} \EE \|\vecg(\hat{\vecx}_i, s_i)\|^2,
\end{align*}
where the first inequality follows due to Jensen's inequality, and the last inequality uses the bound $|S_{r+1}^j| \leq 2(r+1)\tau + 1 \leq 3\tau(r+1)$. \qquad
\end{proof}

\begin{remark}
The $\tau^2$ factor in the upper bound on $\max\EE \|\hat\vecx_k - \vecx_j\|^2$ in Lemma~\ref{lemma:recursions ascd} might be loose. 
We believe that it should instead be $\tau$, when $\tau$ is smaller than some measure of the sparsity. 
If the sparsity of the steps $\vecg(\hat\vecx_i, s_i)$ 
can be exploited, we suspect that the condition $\tau= \mathcal{O}(\sqrt[6]{d})$ in Theorem \ref{theo:ASCD} could be improved to  $\tau= \mathcal{O}(\sqrt[4]{d})$.
\end{remark}

Let us now define for simplicity
$
G_r =  \max_{k\in\setS_r^j}\EE\|\vecg(\hat\vecx_{k},s_k)\|^2\;\;\text{ and
}\;\;\Delta_r = \max_{k \in \setS_r^j} \EE \|\hat\vecx_k - \vecx_j\|^2.$
Observe that that all gradient norms can be bounded as
$$
G_r = \max_{k\in\setS_r^j}\EE\|\vecg(\hat\vecx_{k},s_k)\|^2 = d\max_{k\in\setS_r^j}\EE\|\nabla f(\hat\vecx_{k})\|^2 \leq d \max_{0 \leq k \leq T}\EE\|\nabla f(\hat\vecx_{k})\|^2 = dM^2,
$$
a property that we will use in our bounds.
Observe that $R_0^j = \EE\|\vecg(\hat\vecx_{j},s_j)\|^2 = G_0$ and $R_1^j = \EE \|\hat\vecx_j - \vecx_j\|^2=\Delta_0$.
To obtain bounds for our first two error terms, $R_0^j$ and $R_1^j$, we will expand the recursive relations that are implied by Lemma~\ref{lemma:recursions ascd}.
As shown in \S~\ref{sec:ascd_proofs} of the Appendix, we obtain the following bounds.

\begin{restatable}{lemma}{lemRzeroone}
\label{lemma:R01 bound ascd}
Let  $\tau \leq \frac{\kappa\sqrt{d}}{\ell}$ and set $\gamma = \frac{\theta}{6d L \kappa}$, for any  $\theta \leq 1$ and $\ell \geq 1$.  
Then,
\begin{align*}
R_0^j\le \ct\left(dL^2 a_j + \theta^{2\ell}d M^2\right) \text{ and } R_1^j \leq \ct\left(\theta^2 a_j + \theta^{2\ell}\frac{M^2}{L^2}\right).
\end{align*}
\end{restatable}

The Cauchy-Schwartz inequality implies the bound $R_2^j \leq \sqrt{R_0^j R_1^j}$. 
Unfortunately this approach yields a result that can only guarantee convergence up to a factor of $\sqrt{d}$ slower than serial SCD.
This happens because upper bounding the inner product $\langle\hat\vecx_j - \vecx_j, \vecg(\hat\vecx_j, s_j) \rangle$ by $\|\hat\vecx_j - \vecx_j\|\|\vecg(\hat\vecx_j, s_j)\|$ disregards the extreme sparsity of  $\vecg(\hat\vecx_j, s_j)$. 
The next lemma uses a slightly more involved argument to bound $R_2^j$ exploiting the sparsity of the gradient update. 
The proof can be found in Appendix~\ref{sec:ascd_proofs}.

\begin{restatable}{lemma}{lemRtwo} 
\label{lemma:R2 ascd bound}
Let  $\tau \leq \frac{\kappa\sqrt{d}}{\ell}$ and $\tau = \mathcal{O}(\sqrt[6]{d})$. Then,
$
R_2^j \leq \ct \left(\theta  m a_j  + \theta^{2\ell}\frac{M^2}{L\kappa}\right).$
\end{restatable}

\subsubsection{Putting it all together}

We can now plug in the upper bounds on $R_0^j$, $R_1^j$, and $R_2^j$ in our perturbed iterate recursive formula
$$a_{j+1} 
\leq \left(1- \gamma m\right) a_j  + \gamma^2\underbrace{\EE\|g(\hat\vecx_j,s_j)\|^2}_{R_0^j} + 2\gamma m \underbrace{\Exp\|\hat\vecx_j - \vecx_j\|^2}_{R_1^j} +  2\gamma \underbrace{\EE \langle \hat\vecx_j - \vecx_j, \vecg(\hat\vecx_j,s_j)\rangle}_{R_2^j},$$
to find that ASCD satisfies 
\begin{align*}
a_{j+1}\! \leq\! \biggl( \!1\! -\! \gamma m + \underbrace{\ct\biggl(\gamma^2 dL^2 + \gamma  m\theta^2 + \gamma \theta m\biggr)}_{=r(\gamma)}\biggr)a_j\! +\! \ct\!\biggl(\underbrace{\gamma^2 \theta^{2\ell}dM^2\! +\! \gamma \theta^{2\ell}\frac{M^2}{L\kappa}}_{=\delta(\gamma)}\biggr).
\end{align*}

Observe that in the serial case of SCD the errors $R_1^j$ and $R_2^j$ are zero, and $R_0^j=\EE\|g(\vecx_j,s_j)\|^2$.
By applying the Lipschitz assumption on $f$, we get $\EE\|g(\vecx_j,s_j)\|^2\le dL^2a_j$, and obtain the simple recursive formula
\begin{align}
a_{j+1} \le (1-  \gamma m+\gamma^2dL^2)a_j.
\label{eq:serial_SCD}
\end{align}
To guarantee that ASCD follows the same recursion,
{\it i.e.}, it has the same convergence rate as the one implied by Eq.~\eqref{eq:serial_SCD}, we require that
$\gamma m-r(\gamma) \ge C(\gamma m-\gamma^2dL^2),$
where $C<1$ is a constant.
Solving for $\gamma$ we get
\begin{align*}
&\gamma m-C'\left(\gamma^2 dL^2 + \gamma  m\theta^2 + \gamma \theta m\right) \geq
C (\gamma m-\gamma^2 dL^2) \\
\Leftrightarrow&
(1-C)\gamma m-(C'-C)\gamma^2 dL^2 + C'(\gamma  m\theta^2 + \gamma \theta m) \geq 0\\
\Leftrightarrow&
(C'-C)\gamma dL^2\leq [(1-C)+C'(\theta^2 + \theta)]m 
\Leftrightarrow
\gamma \leq \ct \frac{\theta m}{ dL^2} = \ct \frac{\theta}{ d\kappa L},
\end{align*}
where $C'>1$ is some absolute constant.
For $\gamma = \ct \frac{\theta}{ d\kappa L}$, the $\delta(\gamma)$ term in the recursive bound becomes
\begin{align*}
\ct\!\left(\! \frac{\theta^2}{d^2\!\kappa^2\!L^2} \theta^{2\ell}dM^2\! +\! \frac{\theta}{ d\kappa L} \theta^{2\ell}\frac{M^2}{L\kappa}\!\right)\!
&=\!
\ct\!\left(\! \theta^{2\ell+2} \frac{M^2}{ d\kappa^2\! L^2\!}\! +\! \theta^{2\ell+1}\frac{M^2}{d\kappa^2\!L^2\!}\!\right)
\!\le\!
\ct\theta^{2\ell}\frac{\hat{a}_0}{ d\kappa^2},
\end{align*}
where we used the inequality $M^2\le L^2 \hat{a}_0$.
Hence, ASCD satisfies 
\begin{align*}
a_{j + 1}& \leq \left(1 - \ct \sfrac{\theta}{d\kappa^2}\right)a_j +   \ct\theta^{2\ell}\frac{\hat{a}_0}{ d\kappa^2} 
 \leq 
\left(1 - \ct \sfrac{\theta}{d\kappa^2}\right)^{j+1}a_0 + \ct  \theta^{2\ell} \hat a_0.
\end{align*}
Let us set $\theta$ to be a sufficiently small constant so that $\ct \frac{\theta}{d\kappa^2}= \frac{1}{d\kappa^2}$ and solve for $\ell$ such that $\ct  \theta^{2\ell}\hat a_0 = \epsilon/2 $.
This gives
$\ell = \ct \log\left(\sfrac{\hat a_0}{\epsilon}\right).$
Our main theorem for ASCD now follows from solving $\left(1 - \sfrac{\ct}{d\kappa^2}\right)^{j+1}a_0=\epsilon/2$ for $j$.


%% file: svrg.tex
\section{Sparse and Asynchronous SVRG}
The SVRG algorithm, presented in \cite{johnson2013accelerating}, is a variance-reduction
approach to stochastic gradient descent with strong theoretical guarantees and empirical
performance.  In this section, we present a parallel, asynchronous and sparse variant of 
SVRG.  We also present a convergence analysis, showing that the analysis proceeds in a 
nearly identical way to that of ASCD.

\subsection{Serial Sparse SVRG}
The original SVRG algorithm of \cite{johnson2013accelerating} runs for a number of epochs; 
the per epoch iteration is given as follows:
\begin{equation}
\label{eq: classic SVRG}
\vecx_{j+1} = \vecx_j - \gamma\left(\vecg(\vecx_j, s_j) - \vecg(\vecy, s_j) + \nabla f(\vecy)\right),
\end{equation}
where $\vecy$ is the last iterate of the previous epoch, and as such is updated at the end of every epoch. Here $f$ is of the same form as in \eqref{eq:decomposable-f}:
\begin{equation*}
f(\vecx) = \frac{1}{n}\sum_{i=1}^n f_{e_i}(\vecx),
\end{equation*}
and $\vecg(\vecx, s_j) = \nabla f_{s_j}(\vecx)$, with hyperedges $s_j\in \mathcal{E}$ sampled uniformly at random. 
As is common in the SVRG literature, we further assume that the individual $f_{e_i}$ terms  are $L$-smooth.
The theoretical innovation in SVRG is having an SGM flavored algorithm, with small amortized cost per iteration, where the variance of the gradient estimate is 
smaller than that of standard SGM.
For a certain selection of learning rate, epoch size, and number of iterations, \cite{johnson2013accelerating} establishes that SVRG attains a linear rate.  

Observe that when optimizing a decomposable $f$ with sparse terms, in contrast to SGM, the SVRG iterates will be dense due to the term $\nabla f(\vecy)$.
From a practical perspective, when the SGM iterates are sparse---the case in several applications \cite{niu2011hogwild}---the cost of writing a sparse update in shared memory is significantly smaller than  applying the dense gradient update term $\nabla f(\vecy)$.
Furthermore, these dense updates will cause significantly more memory conflicts in an asynchronous execution, amplifying the error terms in \eqref{eq:aj_main}, and introducing time delays due to memory contention.

A sparse version of SVRG can be obtained by letting the support of the update be determined by that of $\vecg(\vecy,s_j)$:
\begin{equation}
\vecx_{j+1} = \vecx_j - \gamma \left(\vecg(\vecx_j,s_j) - \vecg(\vecy,s_j) +\matD_{s_j}\nabla f(\vecy)\right)= \vecx_j -\gamma\vecv_j,
\label{eq:serial-sparse-svrg}
\end{equation}
where $\matD_{s_j} = \matP_{s_j}\matD$, and $\matP_{s_j}$ is the projection on the support of $s_j$ and $\matD = \diag\left(p_1^{-1}, \ldots, p_d^{-1}\right)$ is a $d\times d$ diagonal matrix. 
The weight $p_v$ is equal to the probability that index $v$ belongs to a hyperedge sampled uniformly at random from $\mathcal{E}$.
These probabilities can be computed from the right degrees of the bipartite graph shown in Fig.~\ref{fig:bipartite_graph}.   
The normalization ensures that $\EE_{s_j} \matD_{s_j}\nabla f(\vecy) = \nabla f(\vecy)$ and thus that $\EE {\vecv}_{j} = \nabla f(\vecx_j)$. 
We will establish the same upper bound on $\EE \|\vecv_j\|^2$ for sparse SVRG as the one used in \cite{johnson2013accelerating} to establish a linear rate of convergence for dense SVRG. 
As before we assume that there exists a uniform bound $M>0$ such that $\|\vecv_j\| \leq M$. 

\begin{lemma}
\label{lem:SVRGvar}
The variance of the serial sparse SVRG procedure in \eqref{eq:serial-sparse-svrg} satisfies
$$
\EE \|\vecv_j\|^2 \leq 2\EE\|\vecg(\vecx_j,s_j) - \vecg(\vecx^*,s_j)\|^2 + 2\EE \|\vecg(\vecy,s_j) - \vecg(\vecx^*,s_j)\|^2 - 2\nabla f(\vecy)^\top \matD \nabla f(\vecy). 
$$
\end{lemma}
\begin{proof}
By definition $\vecv_j = \vecg(\vecx_j, s_j) - \vecg(\vecy, s_j) + \matD_{s_j}\nabla f(\vecy)$. Therefore
\begin{align*}
\EE \|\vecv_j\|^2 &= \EE\| \vecg(\vecx_j, s_j) - \vecg(\vecy, s_j) + \matD_{s_j}\nabla f(\vecy)\|^2 \\
&\leq 2\EE\| \vecg(\vecx_j, s_j) - \vecg(\vecx^*, s_j)\|^2 + 2\EE\| \vecg(\vecy, s_j) - \vecg(\vecx^*,s_j) - \matD_{s_j} \nabla f(\vecy)\|^2. 
\end{align*}
We expand the second term to find that
\begin{align*}
&\EE\| \vecg(\vecy,s_j) - \vecg(\vecx^*,s_j) - \matD_{s_j} \nabla f(\vecy)\|^2 \\
 &= \EE \|\vecg(\vecy,s_j) - \vecg(\vecx^*,s_j) \|^2 - 2\EE \langle\vecg(\vecy, s_j) - \vecg(\vecx^*, s_j) , \matD_{s_j} \nabla f(\vecy)\rangle + \EE \|\matD_{s_j}\nabla f(\vecy)\|^2.
\end{align*}
Since $\vecg(\vecx,s_j)$ is supported on $s_j$ for all $\vecx$, we have
{\small
\begin{align*}
\EE \langle\vecg(\vecy,s_j) - \vecg(\vecx^*,s_j) , \matD_{s_j}\nabla f(\vecy)\rangle &= \EE \langle\vecg(\vecy,s_j) - \vecg(\vecx^*,s_j) ,\matD \nabla f(\vecy)\rangle=\nabla f(\vecy)^\top\matD \nabla f(\vecy),  
\end{align*}
}where the second equality follows by the property of iterated expectations. The conclusion follows because $\EE \|\matD_{s_j}\nabla f(\vecy)\|^2= \nabla f(\vecy)^\top\matD \nabla f(\vecy)$. \qquad
\end{proof}

Observe that the last term in the variance bound is a non-negative quadratic form, hence we can drop it and obtain the same variance bound as the one obtained in \cite{johnson2013accelerating} for dense SVRG.
This directly leads to the following corollary.
\begin{corollary}
Sparse SVRG admits the same convergence rate upper bound as that of the SVRG of \cite{johnson2013accelerating}.
\end{corollary}

We note that usually the convergence rates for SVRG are obtained for function value differences.
However, since our perturbed iterate framework of \S~\ref{section: noisy input stochastic gradients} is based on iterate differences, we re-derive a convergence bound for iterates.
\begin{lemma}
Let the step size be $\gamma = \frac{1}{4L\kappa}$ and the length of an epoch be $8\kappa^2$. Then,
$
\EE \|\vecy_k - \vecx^*\|^2 \leq 0.75^{k}\cdot \EE\|\vecy_{0} - \vecx^*\|^2,
$
 where $\vecy_k$ is the iterate at the end of the $k$-th epoch. 
\end{lemma}

\begin{proof}
We bound the distance to the optimum after one epoch of length $8\kappa^2$:
{\small
\begin{align*}
\EE&\|\vecx_{j+1}-\vecx^*\|^2 
= 
\EE\|\vecx_{j}-\vecx^*\|^2-2\gamma\EE\langle\vecx_j-\vecx^*,\vecv_{j}\rangle+\gamma^2\EE\|\vecv_j\|^2\\
&\le 
\EE\|\vecx_{j} -  \vecx^*\|^2  -  2\gamma\EE \langle \vecx_j - \vecx^*, \nabla f(\vecx_j)\rangle +  2\gamma^2\EE\|\vecg(\vecx_j,s_j)  -  \vecg(\vecx^*,s_j)\|^2 \\ 
&\quad + 2\gamma^2\EE \|\vecg(\vecy,s_j)  -   \vecg(\vecx^*,s_j)\|^2\\
&\le 
\EE\|\vecx_{j}-\vecx^*\|^2-2\gamma\EE\langle\vecx_j-\vecx^*,\nabla f(\vecx_j)\rangle+2\gamma^2L^2\EE\|\vecx_j-\vecx^*\|^2 + 2\gamma^2L^2\EE\|\vecy-\vecx^*\|^2\\
&\le
(1-2\gamma m + 2\gamma^2L^2)\EE\|\vecx_{j}-\vecx^*\|^2 + 2\gamma^2L^2\EE\|\vecy-\vecx^*\|^2.
\end{align*}}

The first inequality follows from Lemma~\ref{lem:SVRGvar} and an application of iterated expectations to obtain $\EE\langle\vecx_j-\vecx^*,\vecv_{j}\rangle  = \EE\langle\vecx_j-\vecx^*,\nabla f(\vecx_j)\rangle$. The second inequality follows from the smoothness of $\vecg(\vecx,s_j)$, and the third inequality follows since $f$ is $m$-strongly convex.

We can rewrite the inequality as  
$a_{j+1}\le (1-2\gamma m +2\gamma^2L^2)a_j+2\gamma^2L^2a_0$, 
because by construction $\vecy = \vecx_0$.
Let $\gamma = \frac{1}{4L\kappa}$.
Then, $1-2\gamma m + 2\gamma^2 L^2 \leq 1-\frac{1}{4\kappa^2}$ and
$$
\sum_{i=0}^j (1-2\gamma m + 2\gamma^2 L^2)^i \le\sum_{i=0}^j \left(1-\sfrac{1}{4\kappa^2}\right)^i \leq \sum_{i=0}^\infty \left(1-\sfrac{1}{4\kappa^2}\right)^i = 4\kappa^2,
$$
since $\frac{1}{4\kappa^2}\le \frac{1}{4}$.
Therefore
\begin{align*}
a_{j+1}&\le \left(1-\sfrac{1}{4\kappa^2}\right)a_j+2\gamma^2L^2a_0
\le \left(1-\sfrac{1}{4\kappa^2}\right)^{j+1}a_0+\sum_{i=0}^{j-1}\left(1-\sfrac{1}{4\kappa^2}\right)^i \cdot 2\gamma^2L^2a_0\\ 
&\le \left(1-\sfrac{1}{4\kappa^2}\right)^{j+1}a_0+4\kappa^2\gamma^2L^2a_0 = \left[\left(1-\sfrac{1}{4\kappa^2}\right)^{j+1}+\sfrac{1}{4}\right]a_0.
\end{align*}
Setting the length of an epoch to be $j = 2\cdot (4\kappa^2)$ gives us $a_{j+1} \le (1/2+1/4)\cdot a_0 = 0.75 \cdot a_0$, and the conclusion follows. \qquad
\end{proof}

We thus obtain the following convergence rate result: 
\begin{theorem}
Sparse SVRG, with step size $\gamma = \ct \frac{1}{L\kappa}$ and epoch size $S = \ct \kappa^2$, reaches accuracy $\EE \|\vecy_E - \vecx^*\|^2 \leq \epsilon$ after 
$
E =\ct \log\left(\sfrac{a_0}{\epsilon}\right)
$
 epochs, where $\vecy_E$ is the last iterate of the final epoch, and $a_0 = \|\vecx_0 - \vecx^*\|^2$ is the initial distance squared to the optimum. 
\label{theo:serial_svrg}
\end{theorem}

\subsection{\CM{}: Asynchronous Parallel Sparse SVRG}
We now present an asynchronous implementation of sparse SVRG.  This implementation, which we
refer to as \CM{}, is given in Algorithm~\ref{algo:KM}.
\begin{center}
\begin{algorithm}[h]
   \caption{\CM{}}
\begin{algorithmic}[1]
\STATE $\vecx = \vecy = \vecx_0$
\FOR{Epoch = $1:E$}
\STATE Compute in parallel ${\bf z} = \nabla f(\vecy)$
\PWHILE{ number of sampled hyperedges $\le S$}
\STATE sample a random hyperedge $s$
\STATE $[\hat\vecx]_s$ = an inconsistent read of the shared variable $[\vecx]_s$ 
\STATE $[{\bf u}]_s = -\gamma\cdot \left(\nabla f_{s}([\vecx]_s)-\nabla f_{s}([\vecy]_s)- \matD_{s} {\bf z}\right) $ 
\FOR{$v\in s$}
\STATE $[\vecx]_v = [\vecx]_v +[{\bf u}]_v$ \hfill{\color{gray}// atomic write}
\ENDFOR
\ENDPWHILE
\STATE $\vecy = \vecx$
\ENDFOR
\end{algorithmic}
   \label{algo:KM}
 \end{algorithm}
\end{center}
Let $\vecv(\hat\vecx_j,s_j)=\vecg(\hat\vecx_j, s_j) - \vecg(\vecy, s_j) + \matD_{s_j}\nabla f(\vecy)$ be the noisy gradient update vector.
Then, after processing a total of $T$ hyperedges, the shared memory contains:
\begin{equation}
\underbrace{\overbrace{\vecx_0 - \gamma \vecv(\hat\vecx_0, s_0) }^{\vecx_1} - \ldots - \gamma \vecv(\hat\vecx_{T - 1},s_{T-1})}_{\vecx_{T}}.
\end{equation}
We now define the perturbed iterates as   
$
\vecx_{i+1} = \vecx_i - \gamma \vecv(\hat\vecx_i, s_i)
$
for $i = 0,1,..., T-1$, where $s_i$ is the $i$-th uniformly sampled hyperedge.
Since $\EE \vecv(\hat\vecx_j, s_j) = \nabla f(\vecx_j)$, \CM{} also satisfies recursion \eqref{eq:aj_main}:
\begin{align*}
a_{j+1}\le(1-\gamma m) a_j+\gamma^2 \underbrace{\EE\|\vecv(\hat\vecx_j,s_j)\|^2}_{R_0^j} + 2\gamma m \underbrace{\EE\|\hat\vecx_j-\vecx_j\|^2}_{R_1^j}  +2\gamma\underbrace{\EE\langle\hat\vecx_j-\vecx_j,\vecv(\hat\vecx_j,s_j)\rangle}_{R_2^j}.
\end{align*}
To prove the convergence of \CM{} we follow the line of reasoning presented in the previous section. Most of the arguments used here come from a straightforward generalization of the analysis of ASCD.
The main result of this section is given below.
\begin{theorem}
Let the maximum number of samples that can overlap in time with a single sample be bounded as
$$\tau = \mathcal{O}\left(\min\left\{\frac{\kappa}{\log\left(\frac{M^2}{L^2\epsilon}\right)}, \sqrt[6]{\frac{n}{\cdeg}}\right\}\right).$$
Then, \CM{}, with step size  $\gamma =  \ct\frac{1}{L\kappa}$ and epoch size $S = \ct \kappa^2$, attains $\EE \|\vecy_E - \vecx^*\|^2 \leq \epsilon$ after 
$
E =\ct \log\left(\sfrac{a_0}{\epsilon}\right)
$
epochs, where $\vecy_E$ is the last iterate of the final epoch, and $a_0 = \|\vecx_0 - \vecx^*\|^2$ is the initial distance squared to the optimum. 
\label{theo:svrg}
\end{theorem}

We would like to note that the total number of iterations in the above bound is---up to a universal constant---the same as that of serial sparse SVRG as presented in Theorem~\ref{theo:serial_svrg}.
Again, as with \HW{} and ASCD, this implies a linear speedup.

Similar to our ASCD analysis, we remark that between the two bounds on $\tau$, the second one is the more restrictive.
The first one is, up to logarithmic factors, equal to the square root of the total number of iterations per epoch; we expect that the size of the epoch is proportional to $n$, the number of function terms (or data points).
This suggests that the first bound is proportional to $\tilde{\mathcal{O}}(\sqrt{n})$ for most reasonable applications.
Moreover, the second bound is certainly loose; we argue that it can be tightened using a more refined analysis.

\subsection{Proof of Theorem~\ref{theo:svrg}}
It is easy to see that due to Lemma~\ref{lem:SVRGvar} we get the following bound on the norm of the gradient estimate.
\begin{lemma}
\label{lemma: shifted grad bound svrg} 
For any $k$ and $j$ we have
\begin{equation}
\label{eq:shifted grad bound svrg}
\EE\left\|\vecv(\hat\vecx_{k},s_k)\right\|^2 \leq 4L^2\left( a_j + a_0+ \EE\|\vecx_j -\hat{\vecx}_k\|^2\right).
\end{equation}
\end{lemma}
\begin{proof}
Due to Lemma~\ref{lem:SVRGvar} we have $\EE\|\vecv(\hat\vecx_j,s_j)\|^2 \le 2L^2\EE\|\hat\vecx_j-\vecx^*\|^2+2L^2\EE\|\vecy - \vecx^*\|^2$.
Then, using the fact that $\vecy = \vecx_0$  and applying the triangle inequality, we obtain the result. \qquad
\end{proof}

The set $\setS_{r}^j$ is defined as in the previous section:
$ \setS_r^j = \{\max\{j-r\tau,0\},\ldots,j -1, j, j +1, \ldots, \min\{j+r\tau,T\}\}$, and has cardinality at most $2r\tau+1$.
By Assumption~\ref{asm:tau}, there exist diagonal sign matrices $\matS_i^j$ with diagonal entries in $\{-1,0,1\}$ such that 
\begin{equation}
\label{eq: shift decomposition svrg}
\hat\vecx_k - \vecx_j = \gamma\sum_{i \in \setS^j_{\ell+1}} {\bf S}_i^j \vecv(\hat\vecx_i, s_i). 
\end{equation}
This leads to the following lemma.
\begin{lemma}
\label{lemma:recursions svrg}
If $G_r = \max_{k\in\setS_r^j}\EE\left\|\vecv(\hat\vecx_{k},s_k)\right\|^2$ and $\Delta_r = \max_{k \in \setS^j_r} \EE \|\hat\vecx_k - \vecx_j\|^2$, 
\begin{equation}
G_r \leq 4L^2\left( a_j +a_0  +\Delta_r \right) \text{ and } \Delta_r \leq (3\gamma\tau(r+1))^2 G_{r+1}.
\end{equation}
\end{lemma}
\begin{proof}
The proof for the bound on $\Delta_r$ is identical to the proof of Lemma~\ref{lemma:recursions ascd}. We then use Lemma~\ref{lemma: shifted grad bound svrg} to bound $\EE \left\|\vecv(\hat\vecx_{k},s_k)\right\|^2$. \qquad 
\end{proof}

As explained in the remark after Lemma~\ref{lemma:recursions ascd}, it should be possible to improve $\tau^2$ to $\tau$ in the upper bound on $\Delta_r$. Doing so would improve the condition $\tau = \mathcal{O}(\sqrt[6]{\sfrac{n}{\cdeg}})$ of Theorem~\ref{theo:svrg} to $\tau = \mathcal{O}(\sqrt[4]{\sfrac{n}{\cdeg}})$. One possible approach to this problem can be found in \S~\ref{sec:svrgR2} of the Appendix.

We can now obtain bounds on the errors due to asynchrony. 
The proofs for the following two lemmas can be found in Appendix~\ref{sec:svrg_proofs}.
\begin{restatable}{lemma}{svrgRzeroone}
Suppose $\tau \leq \frac{\kappa}{\ell}$ and $\gamma = \frac{\theta}{12L\kappa}$. Then the error terms $R_0^j$ and $R_1^j$ of \CM{} satisfy the following inequalities:
\begin{equation}
R_0^j \leq \ct\left(L^2(a_j +a_0)+ \theta^{2\ell}M^2\right)\;\;\text{ and }\;\;R_1^j \leq \ct\left(\theta^2(a_j + a_0) + \theta^{2\ell}\sfrac{M^2}{L^2}\right).
\nonumber
\end{equation}
\end{restatable}
Similarly to the ASCD derivations, we obtain the following bound for $R_2^j$.
\begin{restatable}{lemma}{svrgRtwo}
\label{lemma:R2 svrg bound}
Suppose $\tau \leq \frac{\kappa}{\ell}$ and $\tau = \mathcal{O}\left(\sqrt[6]{\frac{n}{\cdeg}}\right)$, and let $\gamma = \frac{\theta}{12L\kappa}$. Then,
$$
R_2^j \leq \ct\left(\theta \cdot m \cdot(a_j + a_0) + \theta^{2\ell}\frac{M^2}{L\kappa}\right).
$$
\end{restatable}

\subsubsection{Putting it all together}
After plugging in the upper bounds on $R_0^j$, $R_1^j$, and $R_2^j$ in the main recursion satisfied by $\CM{}$, we find that: 
\begin{align*}
a_{j+1} \leq& \left(1 - \gamma m + \ct\left(\gamma^2  L^2 + \gamma m\theta^2 + \gamma \theta m\right) \right)a_j + \\
&+  \ct\left(\gamma^2  L^2 + \gamma m\theta^2 + \gamma \theta m \right)a_0+\gamma^2\ct\theta^{2\ell}M^2 + \gamma\ct \theta^{2\ell}\frac{M^2}{L\kappa}.
\end{align*} 
If we set $\gamma = \ct \frac{\theta}{L\kappa}$, {\it i.e.}, the same step size as serial sparse SVRG (Theorem~\ref{theo:serial_svrg}), then the above becomes
\begin{align*}
a_{j+1} &\leq \left(1  -\ct \frac{\theta}{\kappa^2}\right)a_j +\ct \frac{\theta^2}{\kappa^2} a_0 + \theta^{2\ell + 1}\ct \frac{M^2}{L^2\kappa^2}\\
&\leq \left[\left(1  -\ct \frac{\theta}{\kappa^2}\right)^{j+1} +\ct \theta \right]a_0 + \theta^{2\ell}\ct \frac{M^2}{L^2}.
\end{align*}
We choose $\theta = \ct \leq 1/2$ to be a sufficiently small constant, so that the term $\ct \theta$ in the brackets above is at most $0.5$. 
Then we can choose $j = \ct \kappa^2$ so that the entire coefficient in the brackets is at most $0.75$.
Finally, we set 
$
\ell = \ct \log\left(\sfrac{M^2}{L^2\epsilon}\right)
$, 
so that the last term is smaller than $\epsilon/8$. Let $\vecy_k$ be the iterate after the $k$-th epoch and $A_k = \EE\|\vecy_k - \vecx^*\|^2$. Therefore, \CM{} satisfies the recursion
$$
A_{k+1} \leq 0.75\cdot A_k + \frac{\epsilon}{8} \leq (0.75)^{k+1}A_0 + \frac{\epsilon}{2}.
$$
This implies that $\ct \log(a_0/\epsilon)$ epochs are sufficient to reach $\epsilon$ accuracy, where $a_0$ is $\|\vecx_0 - \vecx^*\|^2$ the initial distance squared to the optimum.


%% file: experiments.tex
\makesavenoteenv{tabular}
\makesavenoteenv{table}
\begin{table}
\centering\footnotesize
\begin{tabular}{|c|c|c|c|c|}\hline
Problem & Dataset & \# data & \# features & sparsity \\\hline\hline
Linear regression
    & Synthetic & 3M & 10K & 20 \\\hline
\multirow{3}{*}{Logistic regression}
    & Synthetic & 3M & 10K & 20 \\\cline{2-5}
    & rcv1 \cite{lewis2004rcv1} from \cite{libsvmdata} & $\approx$ 700K & $\approx$ 42K & $\approx$ 73 \\\cline{2-5}
    & url \cite{ma2009identifying} from \cite{libsvmdata} & $\approx$ 3.2M & $\approx$ 2.4M & $\approx$ 116 \\\hline
\multirow{2}{*}{Vertex cover\footnote{
Following \cite{liu2014asynchronous1}, we optimize a quadratic penalty relaxation for vertex cover
   $\min_{x \in[0,1]^{|V|+|E|}}
\sum_{v \in V} x_v + \frac{\beta}{2} \sum_{(u,v)\in E} (x_u + x_v - x_{u,v} - 1)^2 + \frac{1}{2\beta}\sum_{v\in V} x_v^2 + \sum_{e\in E} x_e^2$.}
}
    & eswiki-2013 \cite{BoVWFI, BRSLLP, BCSU3} & $\approx$ 970K & $\approx$ 23M & $\approx$ 24 \\\cline{2-5}
    & wordassociation-2011 \cite{BoVWFI, BRSLLP, BCSU3} & $\approx$ 10.6K & $\approx$ 72K & $\approx$ 7 \\\hline
\end{tabular}
\caption{
The problems and data sets used in our experimental evaluation.
We test \CM{}, dense SVRG, and \HW{} on three different tasks: linear regression, logistic regression, and vertex cover.
We test the algorithms on sparse data sets, of various sizes and feature dimensions. 
}
\label{tab:probdata} 
\end{table}

\section{Empirical Evaluation of \CM{}}
In this section we evaluate \CM{} empirically.
Our two goals are to demonstrate that (1) \CM{} is faster than dense SVRG, and (2) \CM{} has speedups comparable to those of \HW{}.
We implemented \HW{}, asynchronous dense SVRG, and \CM{} in Scala, and tested them on the problems and datasets listed in Table \ref{tab:probdata}.
Each algorithm was run for 50 epochs, using up to 16 threads.
For the SVRG algorithms, we recompute $\vecy$ and the full gradient $\nabla f(\vecy)$ every two epochs.
We normalize the objective values such that the objective at the initial starting point has a value of one, and the minimum attained across all algorithms and epochs has a value of zero.
Experiments were conducted on a Linux machine with 2 Intel Xeon Processor E5-2670 (2.60GHz, eight cores each) with 250Gb memory.

\begin{figure}[h!]
\centering
    \begin{subfigure}[b]{0.49\columnwidth}
	\centerline{\includegraphics[width = 0.85\columnwidth, trim={0 0.1cm  0 0}, clip]{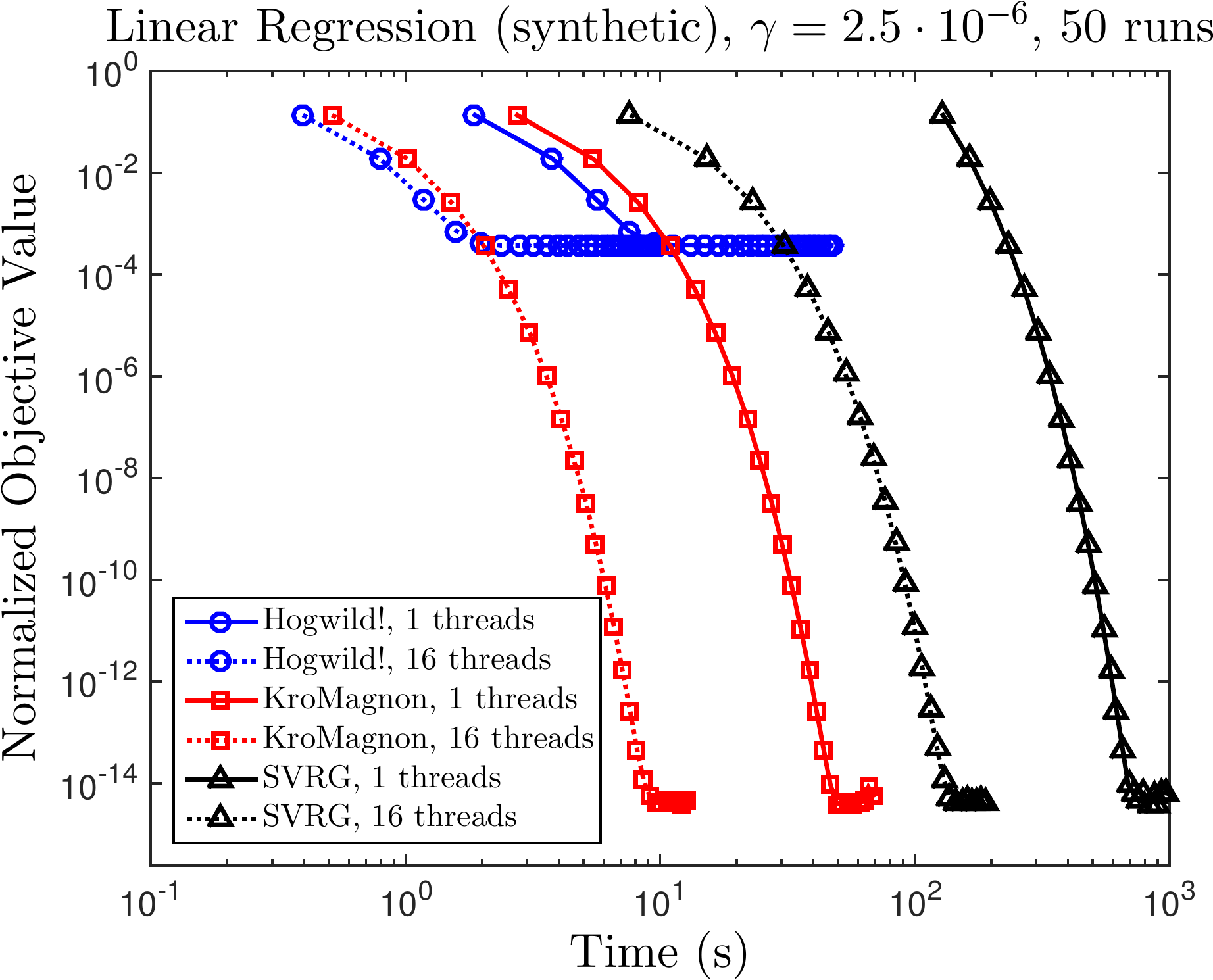}}
      \caption{\scriptsize Linear regression, synthetic}
      \label{fig:convlinsyn}
    \end{subfigure}
   \begin{subfigure}[b]{0.49\columnwidth}
      	\centerline{\includegraphics[width = 0.85\columnwidth, trim={0 0.1cm  0 0}, clip]{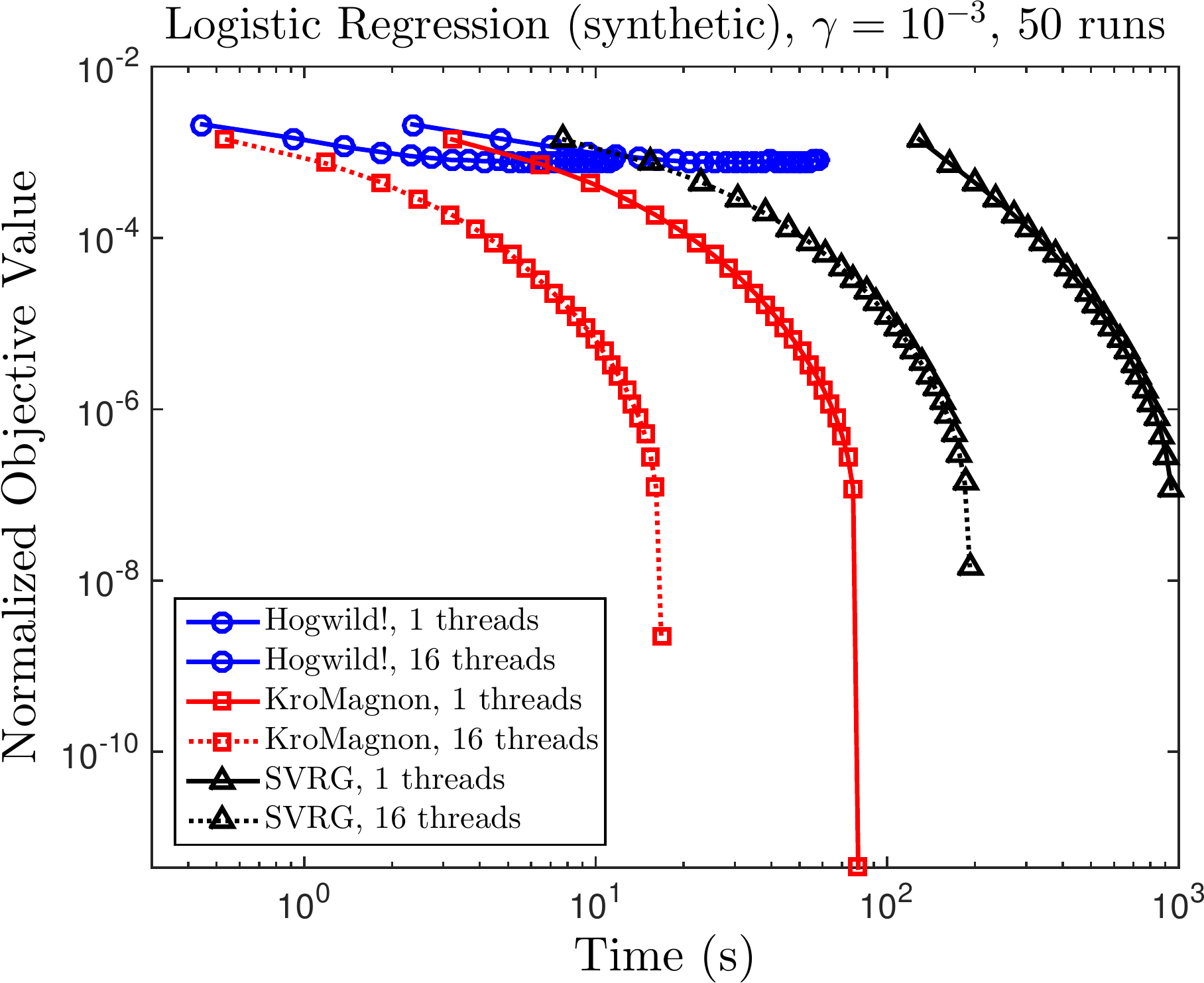}}
      \caption{\scriptsize Logistic regression, synthetic}
      \label{fig:convlogsyn}
    \end{subfigure}

    \begin{subfigure}[b]{0.49\columnwidth}
     	\centerline{\includegraphics[width = 0.85\columnwidth, trim={0 0.1cm  0 0}, clip]{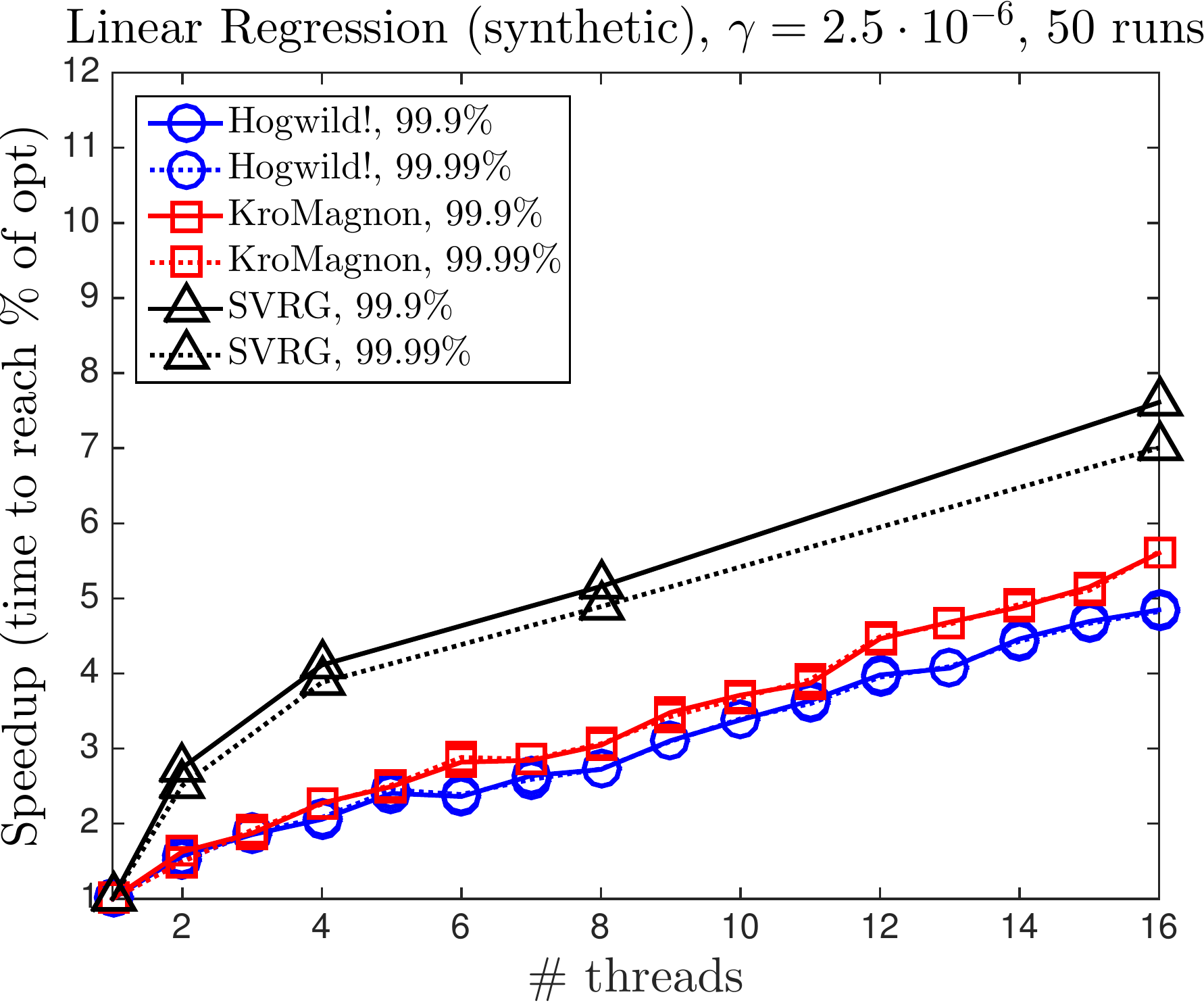}}
      \caption{\scriptsize Linear regression, synthetic}
      \label{fig:spuplinsyn}
    \end{subfigure} 
    \begin{subfigure}[b]{0.49\columnwidth}
     \centerline{\includegraphics[width = 0.85\columnwidth, trim={0 0.1cm  0 0}, clip]{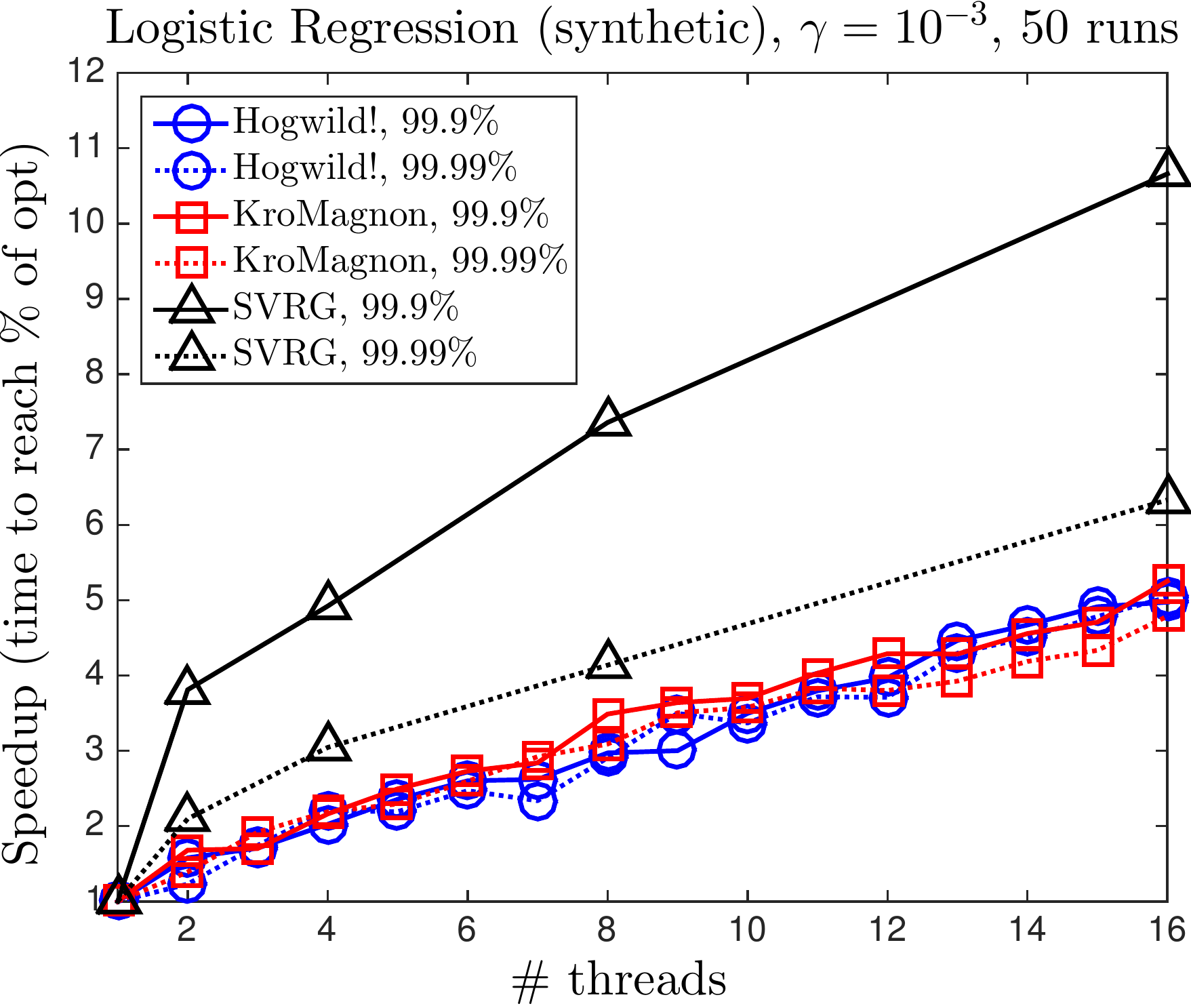}}
      \caption{\scriptsize Logistic regression, synthetic}
      \label{fig:spuplogsyn}
    \end{subfigure}
  \caption{Linear and logistic regression on synthetic data.
  In subfigures (a) and (b) we plot the convergence with respect to normalized objective value as a function of wall-clock time, and in (c) and (d) the speedup with respect to the number of threads.
  The above experiments are all for linear and logistic regression problems on synthetic data, in which we have $3$ million data points, each with $10$K features, and each data point with $20$ nonzero entries.
  We observe that \CM{} is significantly faster than parallel and dense SVRG, while they both can attain better objective values compared to constant step-size \HW{}.
  Moreover, we observe that the speedup gains of \HW{} and \CM{} are scaling reasonably well for up to 16 threads.}
  \label{fig:syn}
\end{figure}

\begin{figure}[h!]
\centering
    \begin{subfigure}[b]{0.49\columnwidth}
	\centerline{\includegraphics[width = 0.85\columnwidth, trim={0 0.1cm  0 0}, clip]{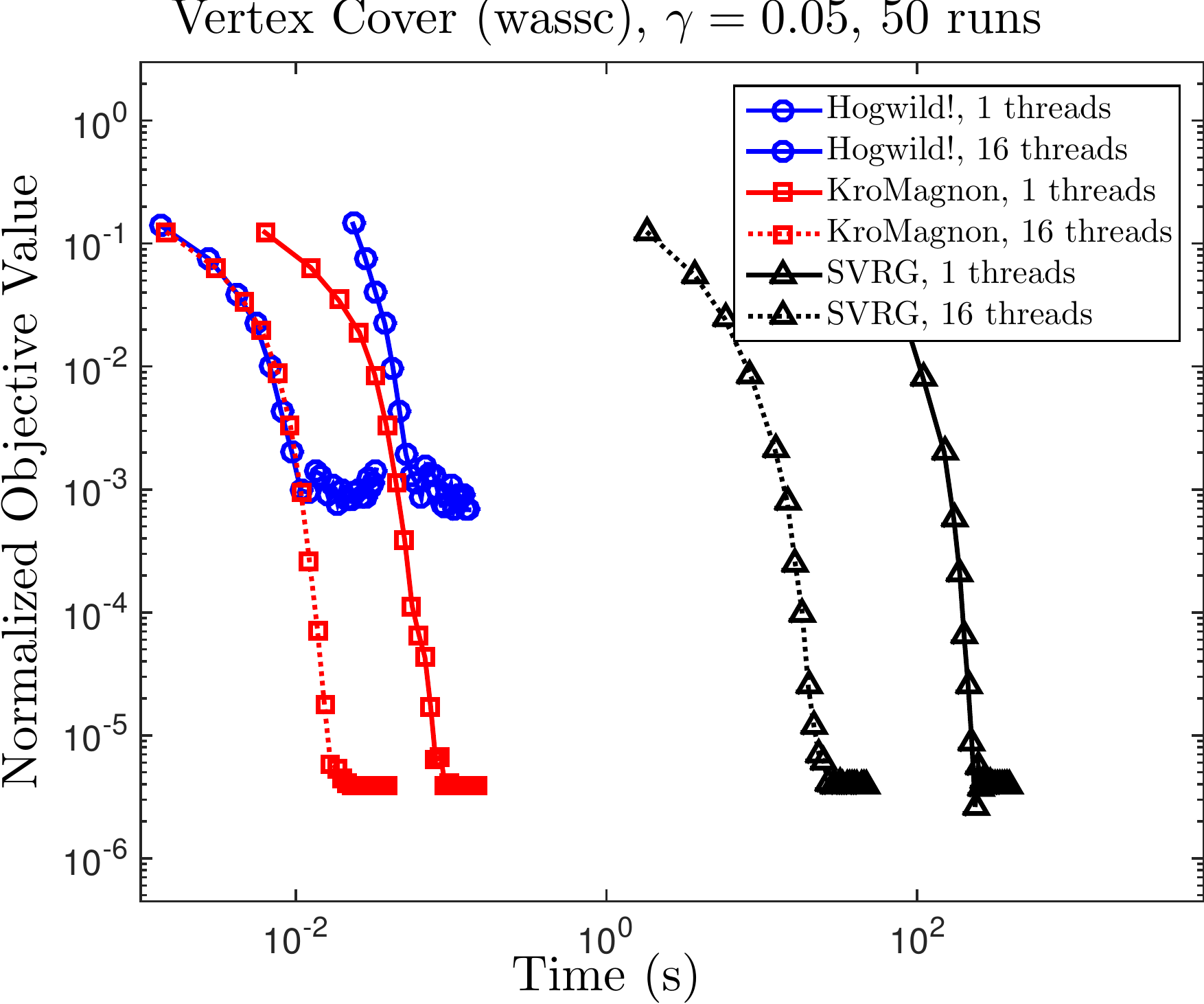}}
\caption{\scriptsize Vertex cover, wordassociation}
      \label{fig:convvcvwas}
    \end{subfigure}
   \begin{subfigure}[b]{0.49\columnwidth}
      	\centerline{\includegraphics[width = 0.85\columnwidth, trim={0 0.1cm  0 0}, clip]{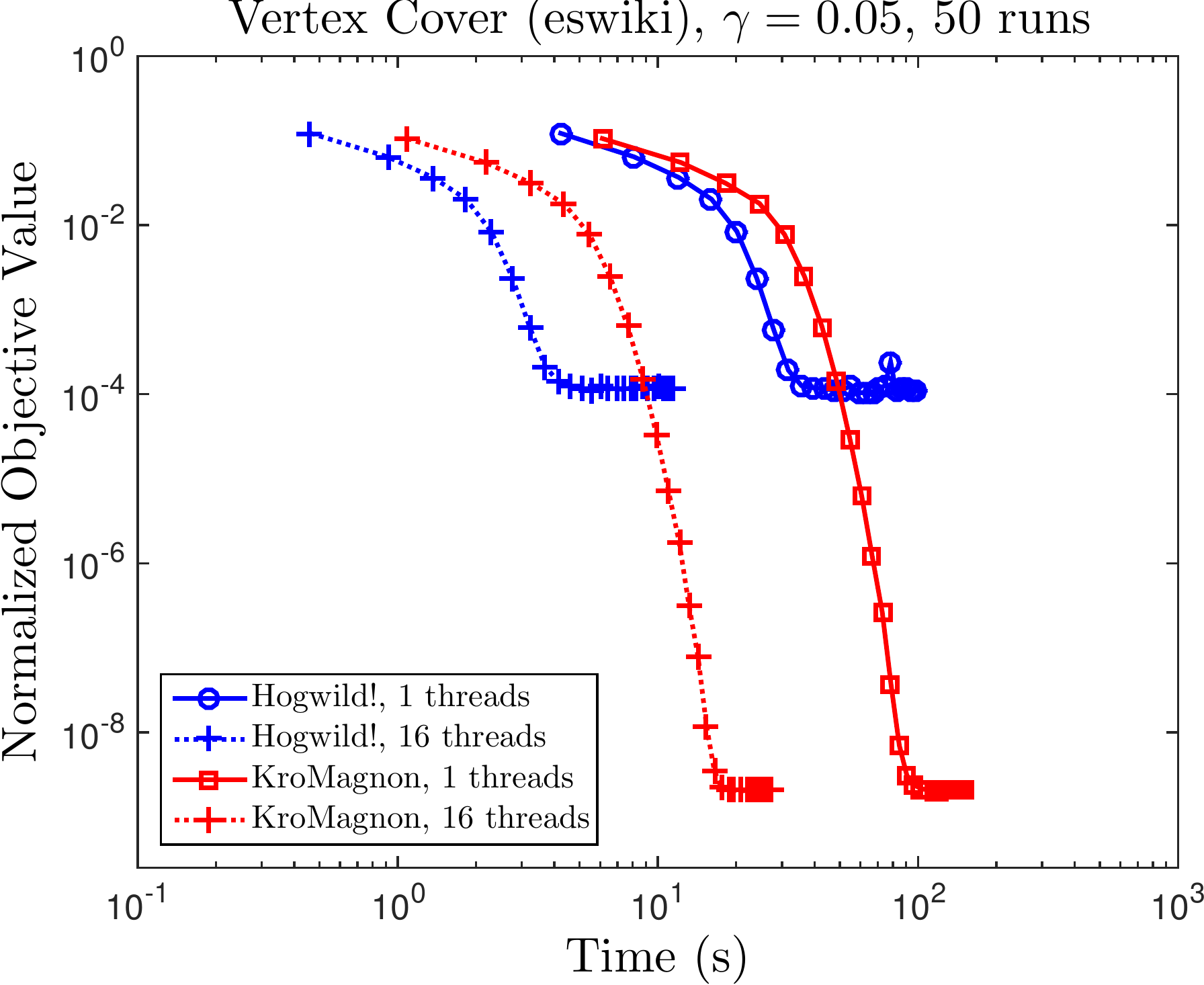}}
  \caption{\scriptsize Vertex cover, eswiki-2013}
      \label{fig:convvcvesw}
    \end{subfigure}

    \begin{subfigure}[b]{0.49\columnwidth}
	\centerline{\includegraphics[width = 0.85\columnwidth, trim={0 0.1cm  0 0}, clip]{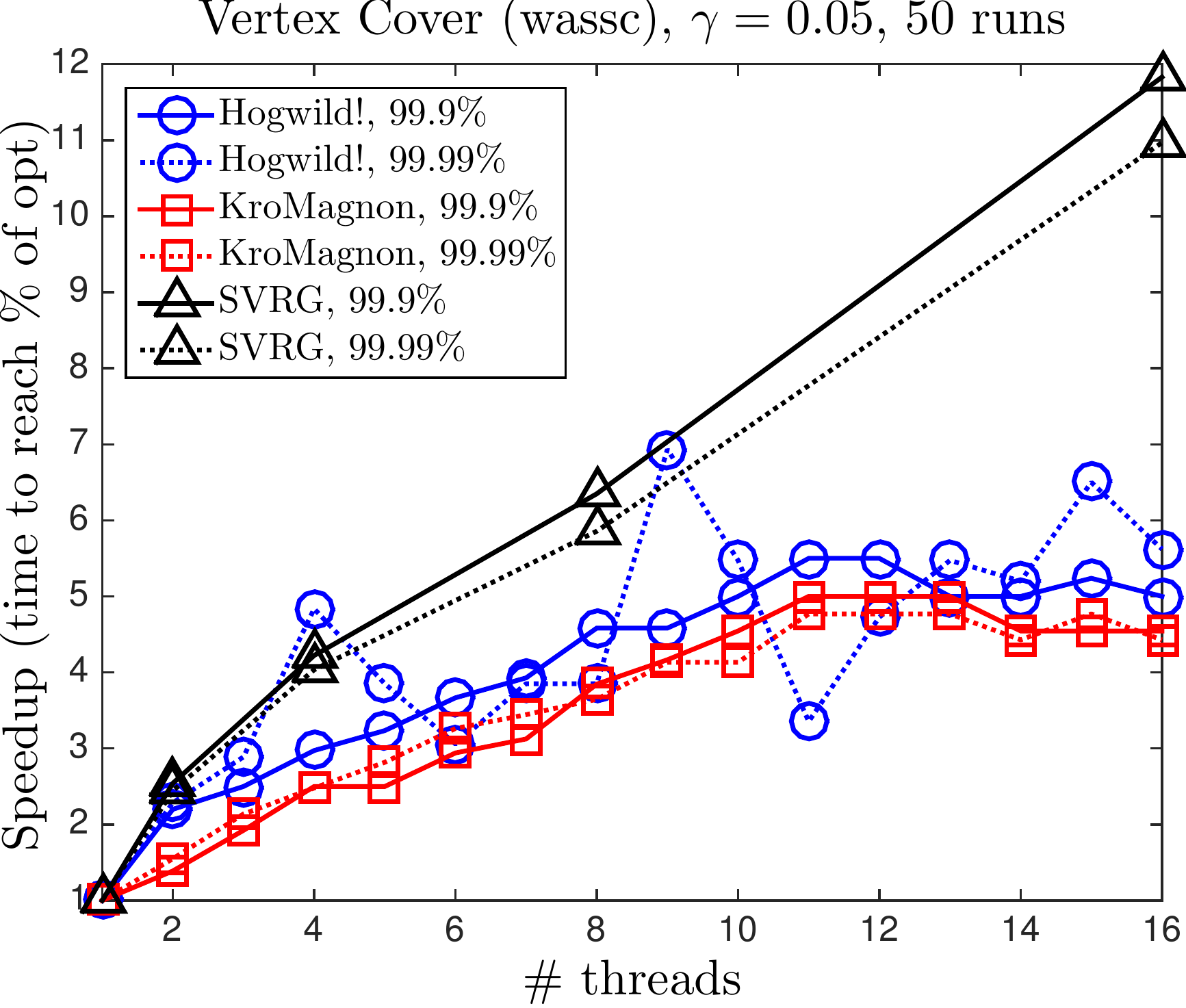}}
\caption{\scriptsize Vertex cover, wordassociation}
      \label{fig:spupvcvwas}
    \end{subfigure}
   \begin{subfigure}[b]{0.49\columnwidth}
      	\centerline{\includegraphics[width = 0.85\columnwidth, trim={0 0.1cm  0 0}, clip]{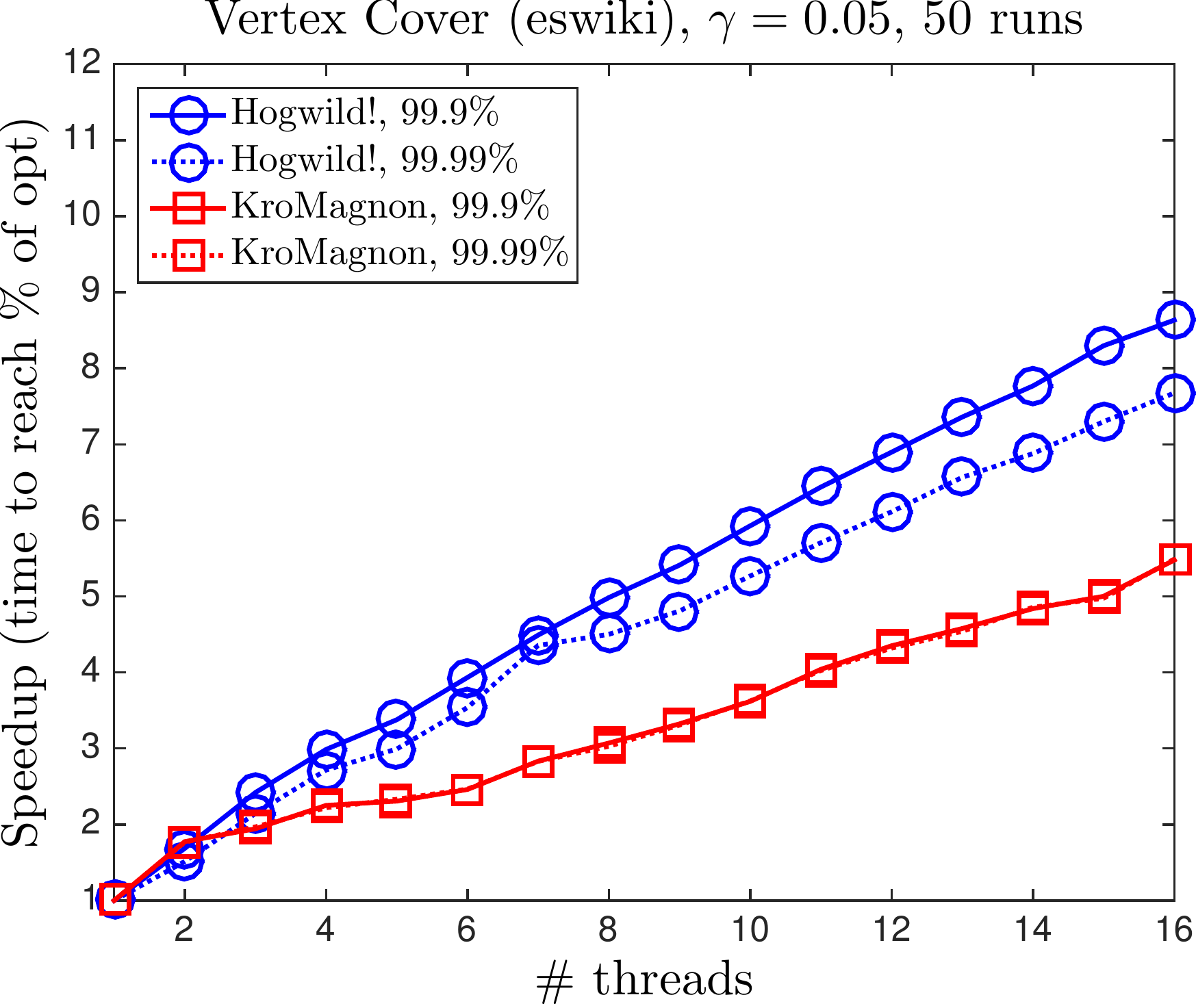}}
  \caption{\scriptsize Vertex cover, eswiki-2013}
      \label{fig:spupvcvesw}
    \end{subfigure}
  \caption{Vertex cover on the wordassociation-2011 and eswiki-2013 datasets.
  Subfigure (a) shows the convergence of the algorithms on wordassociation-2011, a small graph with less than 11,000 vertices.
  \CM{} on a single thread is 3-4 orders of magnitude faster than desnse SVRG on this dataset.
  Convergence of \CM{} and \HW{} on the eswiki-2013 dataset is shown in subfigure (b); we were unable to run dense SVRG on this larger graph.
  Subfigures (c) and (d) show the speedups of the algorithms on the two datasets.
  In subfigure (c), both \HW{} and \CM{} exhibit poorer speedups than dense SVRG because of the rapid conve on the smaller wordassociation-2011 dataset.
 In subfigure (d) we observe that \HW{} achieves a speedup of up to 8x and \CM{} up to 5x. 
  }
  \label{fig:data}
\end{figure}

\paragraph{Comparison with dense SVRG}
We were unable to run dense SVRG on the url and eswiki-2013 datasets due to the large number of features.
Figures \ref{fig:convlinsyn}, \ref{fig:convlogsyn}, and \ref{fig:convlogrcv} show that \CM{} is one-two orders of magnitude faster than dense SVRG.
In fact, running dense SVRG on 16 threads is slower than \CM{} on a single thread.
Moreover, as seen in Fig. \ref{fig:convvcvwas}, \CM{} on 16 threads can be up to four orders of magnitude faster than serial dense SVRG.
Both dense SVRG and \CM{} attain similar optima.
\paragraph{Speedups}
We measured the time each algorithm takes to achieve 99.9\% and 99.99\% of the minimum achieved by that algorithm.
Speedups are computed relative to the runtime of the algorithm on a single thread.
Although the speedup of \CM{} varies across datasets, we find that \CM{} has comparable speedups with \HW{} on all datasets, as shown in Figure \ref{fig:spuplinsyn}, \ref{fig:spuplogsyn}, \ref{fig:spupvcvwas}, \ref{fig:spupvcvesw},   \ref{fig:spuplogrcv}, \ref{fig:spuplogurl}. We further observe that dense SVRG has better speedup scaling. This happens because the per iteration complexity of \HW{} and \CM{} is significantly cheaper to the extent that the additional overhead associated with having extra threads leads to some speedup loss; this is not the case for dense SVRG as the per iteration cost is higher.

\begin{figure}[h!]
\centering
  \begin{subfigure}[b]{0.49\columnwidth}
	\centerline{\includegraphics[width = 0.85\columnwidth, trim={0 0.1cm  0 0}, clip]{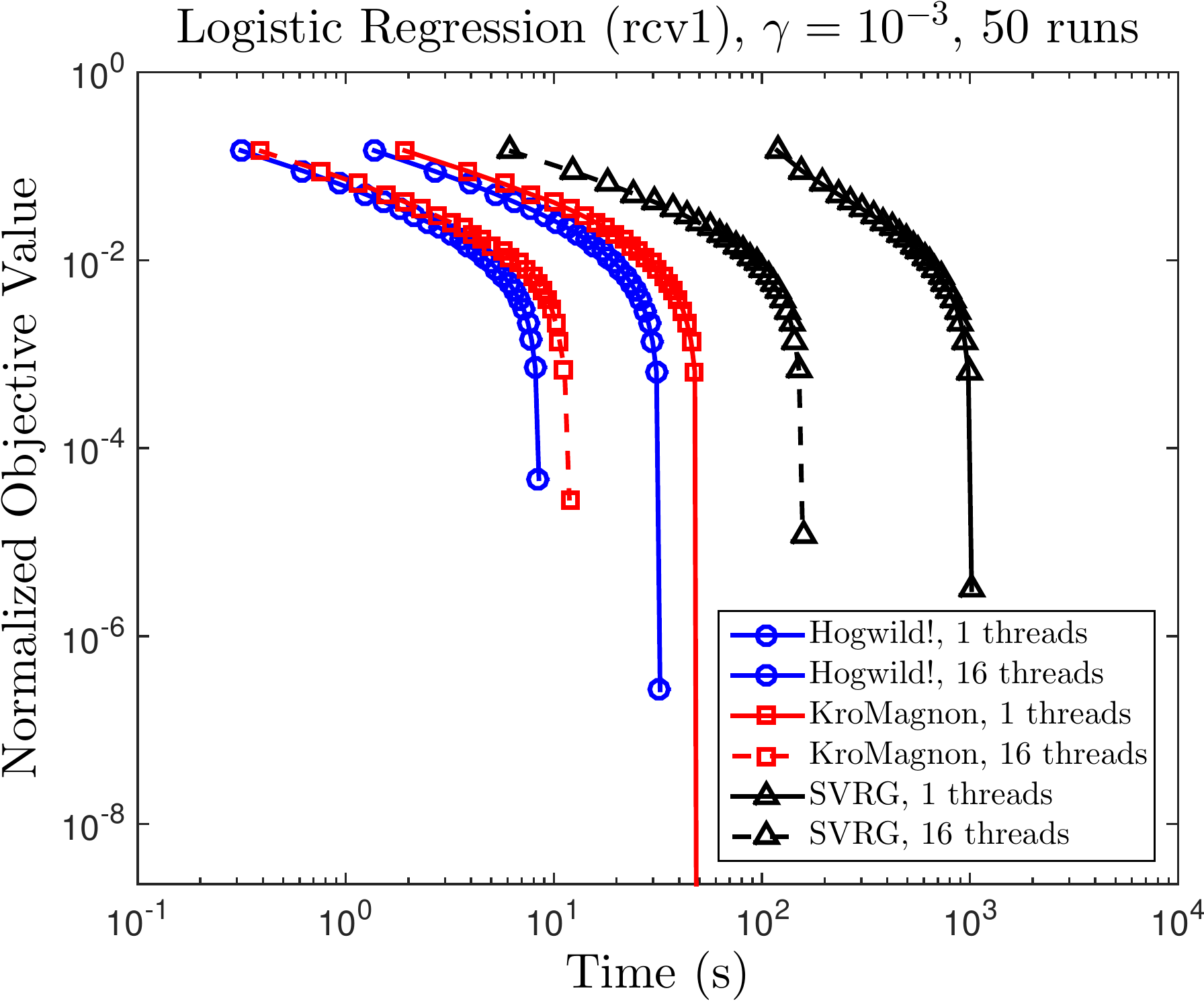}}
      \caption{\scriptsize Logistic regression, rcv1}
      \label{fig:convlogrcv}
    \end{subfigure}
   \begin{subfigure}[b]{0.49\columnwidth}
      	\centerline{\includegraphics[width = 0.85\columnwidth, trim={0 0.1cm  0 0}, clip]{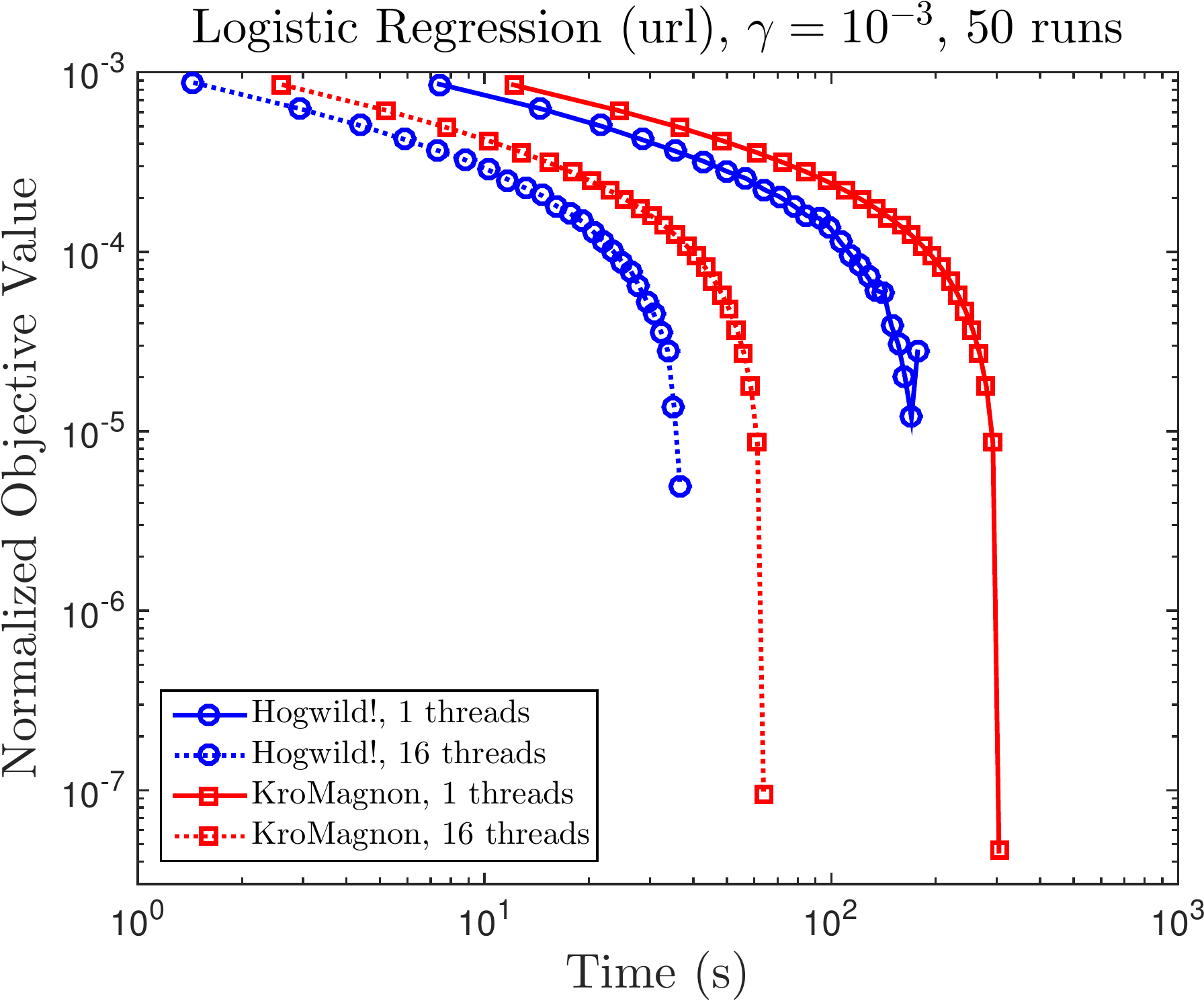}}
      \caption{\scriptsize Logistic regression, url}
      \label{fig:convlogurl}
    \end{subfigure}

      \begin{subfigure}[b]{0.49\columnwidth}
	\centerline{\includegraphics[width = 0.85\columnwidth, trim={0 0.1cm  0 0}, clip]{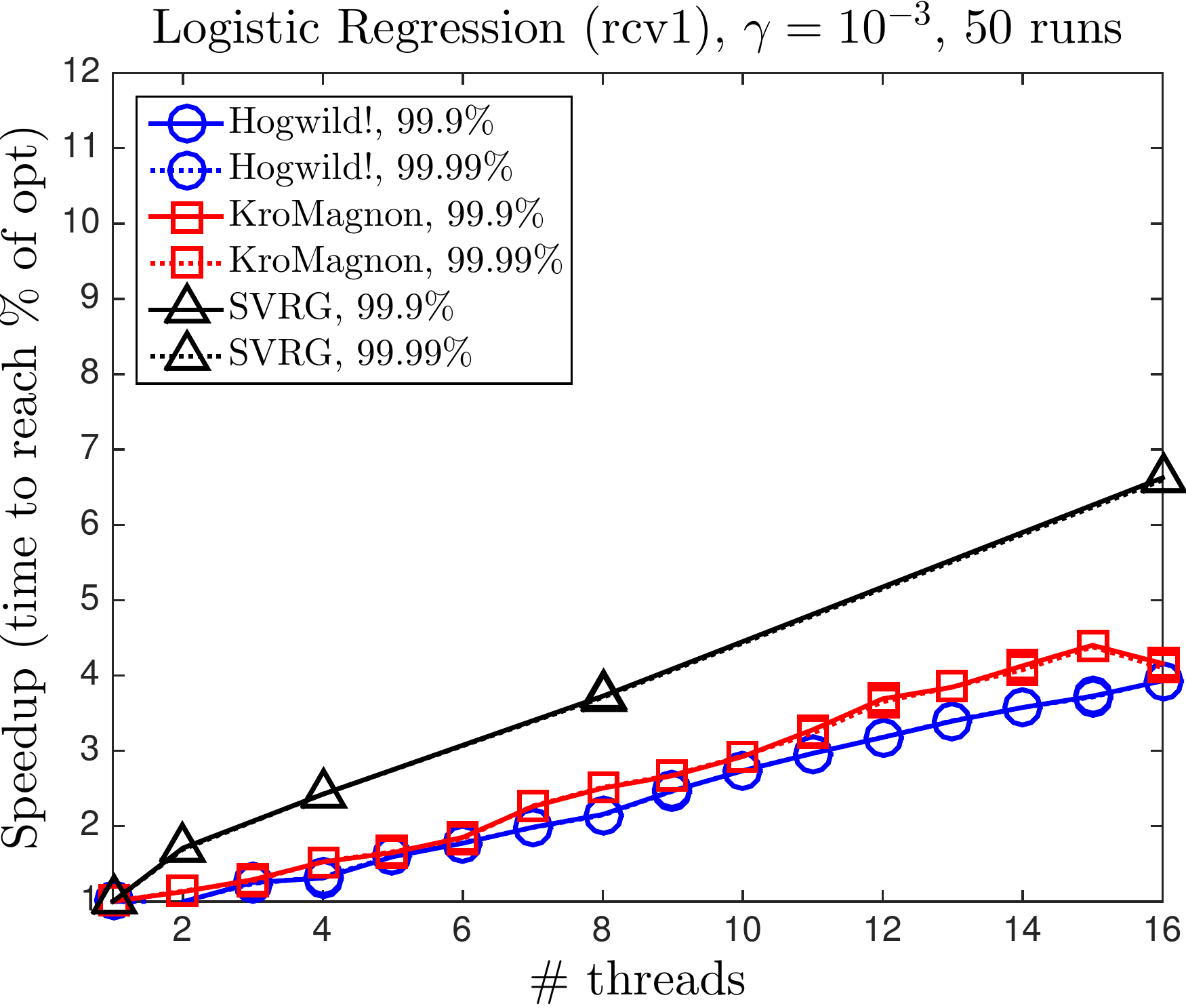}}
      \caption{\scriptsize Logistic regression, rcv1}
      \label{fig:spuplogrcv}
    \end{subfigure}
   \begin{subfigure}[b]{0.49\columnwidth}
      	\centerline{\includegraphics[width = 0.85\columnwidth, trim={0 0.1cm  0 0}, clip]{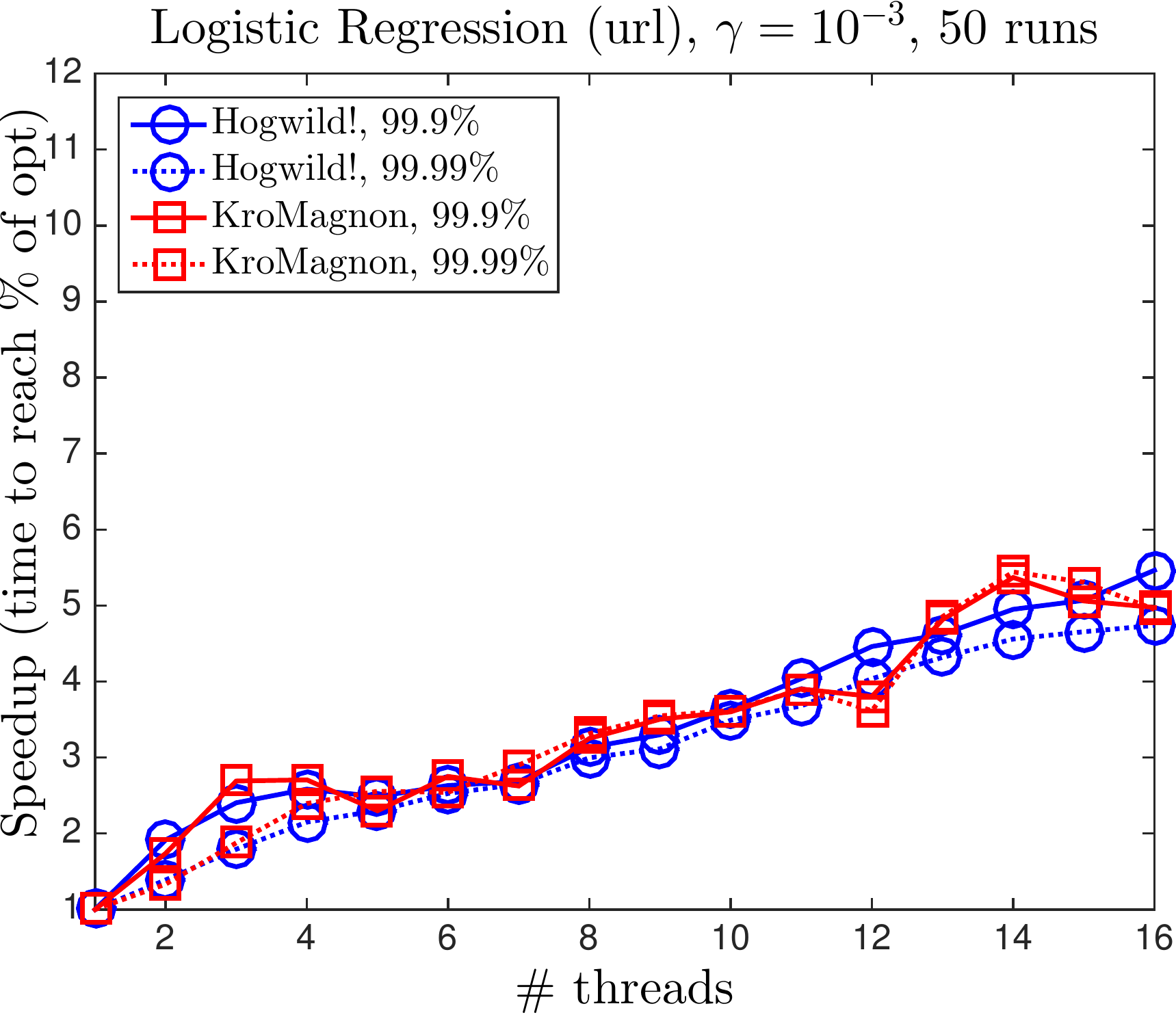}}
      \caption{\scriptsize Logistic regression, url}
      \label{fig:spuplogurl}
    \end{subfigure}
  \caption{Logistic regression on the rcv1 and url datasets.
  Subfigure (a) shows the convergence of the algorithms on the rcv1 dataset.
  For a given objective value, \CM{} is 1-2 orders of magnitude faster than dense SVRG.
  On the larger url dataset (subfigure (b)), we were unable to run dense SVRG.
  Note that some of the effect of asynchrony can be observed in these experiments: the objective values obtained by \CM{}, \HW{}, and dense SVRG are  slightly different on 1 thread compared to 16 threads.
  Speedups of the algorithms are shown in subfigures (c) and (d)--\CM{} has a slightly better speedup than \HW{} on rcv1, and the same speedup on url.
  }
  \label{fig:data2}
\end{figure}


%% file: proofs.tex
\section{Removing the Independence Assumption}
\label{sec:lift_independence}
Our main analysis for \HW{} relied on the independence between $\hat\vecx_i$ and $s_i$ (Assumption~\ref{asm:ind}). 
Here, we show how to lift this assumption, conditional on the fact that each of the $f_{e_i}$ terms is at least $L_{\text{s}}$ smooth.
Observe that the following is no longer true:
$
\EE \langle \hat\vecx_i - \vecx^*, \vecg(\hat\vecx_i, s_i)\rangle = \EE \langle \hat\vecx_i - \vecx^*, \nabla f(\hat\vecx_i)\rangle 
$ 
since we cannot use iterated expectations, precisely due to the lack of independence of the samples and the read variables. 
However, $\overline \vecx_i$, defined in \S~\ref{section: hogwild}, is still independent of $s_i$. 
Therefore, we expand our derivation in \eqref{eq:aj_first} in the following way:
\begin{align*}
a_{j+1} &\leq a_j  - 2\gamma\EE\langle \overline \vecx_i - \vecx^*, \vecg(\overline\vecx_j, s_j)\rangle + \gamma^2\EE\|g(\hat\vecx_j,\xi_j)\|^2 + 2\gamma \EE \langle \overline\vecx_j - \vecx_j, \vecg(\overline\vecx_j,\xi_j)\rangle\\
&\qquad + 2\gamma \EE\langle  \vecx_j -\vecx^*, \vecg(\overline\vecx_j,s_j) - \vecg(\hat\vecx_j,s_j) \rangle.
\end{align*}
Since $\overline\vecx_j$ and $s_j$ are independent by construction, we use iterated expectations to get:
$
\EE\langle \overline \vecx_i - \vecx^*, \vecg(\overline\vecx_j, s_j)\rangle = 
\EE\langle \overline \vecx_i - \vecx^*, \nabla f(\overline\vecx_j)\rangle.
$
As before, the strong convexity of $f$ and the triangle inequality imply that
$
\langle \overline\vecx_j - \vecx^*, \nabla f(\overline\vecx_j)\rangle \geq \frac{m}{2}\|\vecx_j  - \vecx^*\|^2 - m\|\overline\vecx_j - \vecx_j\|^2. 
$
 Putting everything together we get the following recursion for the sequence $a_j$. 
\begin{align*}
a_{j+1} &\leq (1 -\gamma m)a_j + \gamma^2\underbrace{\EE\|g(\hat\vecx_j,\xi_j)\|^2}_{R_0^j} + 2\gamma m\underbrace{\EE\|\overline\vecx_j - \vecx_j\|^2}_{R_1^j} + 2\gamma \underbrace{\EE \langle \overline\vecx_j - \vecx_j, \vecg(\overline\vecx_j,\xi_j)\rangle}_{R_2^j}\\
&\qquad + 2\gamma \underbrace{\EE\langle  \vecx_j -\vecx^*, \vecg(\overline\vecx_j,s_j) - \vecg(\hat\vecx_j,s_j) \rangle}_{R_3^j}.
\end{align*}
The reader can verify that although $R_1^j$ and $R_2^j$ are defined now in terms of $\overline{\vecx}_j$, the upper bounds derived in \S~\ref{section: hogwild} still hold. 
We bound $R_3^j$ as follows
{\small
\begin{align*}
\EE\langle \vecx_j \hspace{-0.03cm}-\hspace{-0.03cm} \vecx^*,\vecg(\overline\vecx_j, s_j) \hspace{-0.03cm}-\hspace{-0.03cm} \vecg(\hat\vecx_j, s_j)\rangle  \hspace{-0.1cm}\leq\hspace{-0.08cm} \EE\|\vecx_j - \vecx^*\|\|\vecg(\overline\vecx_j, s_j)\hspace{-0.03cm}-\hspace{-0.03cm} \vecg(\hat\vecx_j, s_j)\|
\leq \frac{m a_j}{4} \hspace{-0.08cm}+\hspace{-0.08cm} \frac{L_{\text{s}}^2}{m}\EE \|\overline\vecx_j - \hat\vecx_j\|^2,
\end{align*}
}where the last inequality follows by the smoothness of the gradient steps and the arithmetic-geometric mean inequality. Therefore 
{\small
\begin{align*}
a_{j+1} &\leq \left(1 -\frac{\gamma m}{2}\right)a_j + \gamma^2\underbrace{\EE\|g(\hat\vecx_j,\xi_j)\|^2}_{R_0^j} + 2\gamma m\underbrace{\EE\|\overline\vecx_j - \vecx_j\|^2}_{R_1^j} + 2\gamma \underbrace{\EE \langle \overline\vecx_j - \vecx_j, \vecg(\overline\vecx_j,\xi_j)\rangle}_{R_2^j}\\
&\qquad + 2\gamma \frac{L_{\text{s}}^2}{m}\underbrace{\EE \|\overline\vecx_j - \hat\vecx_j\|^2}_{R_4^j}.
\end{align*}
}Since we can upper bound $R_4^j$ by the same bound derived for $R_3^j$, we obtain the following convergence result for \HW{}:
\begin{theorem}
\label{thm:HW2}
If the number of samples that overlap in time with a single sample during the execution of \HW{} is bounded as 
$\tau =\mathcal{O}\left(\min\left\{\sfrac{n}{\cdeg}, \sfrac{M^2}{\epsilon L^2} \right\}\right)$, 
then 
\HW{} with step-size $\gamma = \ct \frac{\epsilon m}{M^2}$, reaches an accuracy of $\EE\|\vecx_k-\vecx^*\|^2\le \epsilon$ after
$
T \geq \ct \frac{M^2 \log\left(\sfrac{a_0}{\epsilon}\right)}{\epsilon m^2}
$
iterations.
\end{theorem}
The only difference between this result and the one in our main section is that we can guarantee speedup for a smaller range of $\tau$. 
Similar ideas could be applied to the analysis of ASCD and \CM{}.

\section{Omitted Proofs}
\subsection{ASCD}
\label{sec:ascd_proofs}
\subsubsection{Bounding $R_0^j$ and $R_1^j$}
\lemRzeroone*
\begin{proof}
Let $ A = 2dL^2 a_j$, $B = 2dL^2$ and $C = (\gamma \tau)^2$. Then, we can rewrite the bounds of Lemma~\ref{lemma:recursions ascd} as
$
G_r \le A +  B\cdot \Delta_r \text{ and } \Delta_r \le 3^2(r+1)^2 \cdot C \cdot G_{r+1}, 
$
which implies $G_r \le A +  3^2(r+1)^2BC\cdot G_{r+1}$.
We can now upper bound $R^j_0 = G_0$, by applying the previous inequality  $\ell$ times.
If we expand the formulas, we get
\begin{align}
R^j_0 = G_0 \leq A\sum_{i=0}^{\ell-1} (3^i \cdot i!)^2 (BC)^i + (3^\ell\cdot \ell!)^2 (BC)^\ell G_\ell.
\label{eq:R0recursive ascd}
\end{align}
Since $\gamma = \frac{\theta}{6dL\kappa}$ and $\tau \leq \frac{\kappa \sqrt{d}}{\ell}$ (the choice of $\tau$ is made so that the sum in \eqref{eq:R0recursive ascd} is significantly small), we have 
$
BC = 2dL^2 \gamma^2\tau^2 \leq 2dL^2 \frac{\theta^2}{6^2 d^2L^2 \kappa^2}\frac{\kappa^2d}{\ell^2} \leq \frac{1}{2\cdot 3^2\cdot\ell^2}.
$
Using the upper bound $k! \le k^k$ on each term of the sum \eqref{eq:R0recursive ascd}, and plugging in the upper bound on $BC$, we get
$$
\sum_{i = 0}^{\ell - 1}(3^i \cdot i!)^2(BC)^i \leq \sum_{i = 0}^{\ell - 1}\frac{(i!)^2}{2^i\ell^{2i}} \leq\sum_{i = 0}^{\ell - 1}\frac{1}{2^i}\left(\sfrac{i}{ \ell}\right)^{2i} \leq 2.
$$
Similarly, we obtain the following upper bound on the last term of Eq.~\ref{eq:R0recursive ascd}
$(3^\ell \cdot \ell!)^{2} (BC)^\ell \leq 2^{-\ell}\theta^{2\ell}.$
Finally, $G_{\ell} \leq dM^2$, and combining the above gives us $R_0^j \leq \ct\left(dL^2 a_j + \theta^{2\ell} dM^2\right).$

We can now bound $R_1^j$.
By definition $R_1^j = \Delta_0$, and from Lemma~\ref{lemma:recursions ascd} we have $\Delta_0 \leq 3^2 \cdot C \cdot G_1$. 
We can bound $G_1$ similarly to $G_0$ as
$$
G_1\leq A\sum_{i=0}^{\ell-1} (3^i \cdot (i+1)!)^2 (BC)^i + (3^\ell\cdot (\ell+1)!)^2 (BC)^\ell G_{\ell + 1}.
$$ 
As before $BC \leq \frac{\theta^2}{2\cdot 3^2 \ell^2}$. Since $(i+1)! \leq 2\ell^{i}$ for any $0\leq i\leq \ell$, it follows that 
$
\sum_{i = 0}^{\ell - 1}(3^i \cdot (i+1)!)^2 (BC)^i\leq \ct,
$
and $(3^\ell\cdot (\ell+1)!)^2 (BC)^\ell \le \ct\theta^{2\ell}$. Therefore, because $G_{\ell + 1} \leq dM^2$, we obtain $G_1 \leq \ct (dL^2 a_j + \theta^{2\ell}dM^2)$. Since $C = (\gamma \tau)^2 \leq \frac{\theta^2}{dL^2}$, it follows that $R_1 \leq \ct\left(\theta^2 a_j + \theta^{2\ell}(M/L)^2\right)$. \qquad
\end{proof}

\subsubsection{Bounding $R_2^j$}

\lemRtwo*
\begin{proof}
From \eqref{eq: shift decomposition ascd} we can upper bound $R_2^j$ as follows. 
\begin{equation}
\label{eq:R2 sum expansion ascd}
R_2^j = \EE \langle\hat\vecx_j-\vecx_j,\vecg(\hat\vecx_j,s_j)\rangle 
\le 
\gamma\cdot \sum_{\substack{i=j-\tau\\ i \ne j}}^{j+\tau} \EE \|\vecg(\hat{\vecx}_i, s_i)\|\cdot \|\vecg(\hat\vecx_j,s_j)\| \cdot \indi{(s_i=s_j)}.
\end{equation}
The random variable  $\indi(s_i = s_j)$ encodes the sparsity of the gradient steps. 
To take advantage of this sparsity we use smoothness to replace the iterates $\hat\vecx_i$ and $\hat\vecx_j$, by $\hat\vecx_{j-3\tau}$.
The latter iterate is independent of both $s_i$ and $s_j$ by our assumption that no more than $\tau$ coordinates can be updated while a core is processing a single coordinate. 
This independence will allows us to ``untangle" the expectation of $\indi(s_i = s_j)$ from the inner products in the above sum, which will result in a significantly improved bound on $R_2^j$ compared to applying Cauchy-Schwartz directly on it.

For clarity, we note that when $j< 3\tau$, we have $\hat\vecx_{j -3\tau} = \vecx_0$.
From the $L$-Lipschitz assumption on the gradient $\nabla f(\vecx)$, we get the following bounds
\begin{align*}
\|\vecg(\hat\vecx_j,s_j)\| &\le \|\vecg(\hat\vecx_{j-3\tau},s_j)\|+dL\|\hat\vecx_{j-3\tau}-\hat\vecx_j\|\\
\|\vecg(\hat\vecx_i,s_i)\| &\le \|\vecg(\hat\vecx_{j-3\tau},s_i)\|+dL\|\hat\vecx_{j-3\tau}-\hat\vecx_i\|.
\end{align*}
Then, the expectation of a term $\|\vecg(\hat{\vecx}_i, s_i)\|\cdot \|\vecg(\hat\vecx_j,s_j)\| \cdot \indi{(s_i=s_j)}$ in the sum \eqref{eq:R2 sum expansion ascd} is upper bounded by
\begin{align*}
\EE\biggl\{
\biggl(&\|\vecg(\hat\vecx_{j-3\tau},s_i)\|\; \|\vecg(\hat\vecx_{j-3\tau},s_j)\| + (dL)^2  \|\hat\vecx_{j-3\tau}-\hat\vecx_i\|\; \|\hat\vecx_{j-3\tau}-\hat\vecx_j\|\\
+&
dL  \|\vecg(\hat\vecx_{j-3\tau},s_j)\|\; \|\hat\vecx_{j-3\tau}-\hat\vecx_i\| + dL  \|\vecg(\hat\vecx_{j-3\tau},s_i)\|\; \|\hat\vecx_{j-3\tau}-\hat\vecx_j\| \biggr)\cdot \indi{(s_i=s_j)}
\biggr\}.
\end{align*}
We first bound the second term using Cauchy-Schwartz and the property of iterated expectation, to exploit the expectation of the $\indi{(s_i=s_j)}$ term
\begin{align*}
\EE\{
\|\hat\vecx_{j-3\tau}-\hat\vecx_i\|&\cdot \|\hat\vecx_{j-3\tau}-\hat\vecx_j\|\cdot \indi{(s_i=s_j)}
\}\\
&\le 
\sqrt{
\EE\{\|\hat\vecx_{j-3\tau}-\hat\vecx_i\|^2\}\cdot 
\EE\{\|\hat\vecx_{j-3\tau}-\hat\vecx_j\|^2\cdot \indi{(s_i=s_j)}\} }\\
&= 
\sqrt{
\EE\{\|\hat\vecx_{j-3\tau}-\hat\vecx_i\|^2\}\cdot 
\EE_{\sim s_j}\{\|\hat\vecx_{j-3\tau}-\hat\vecx_j\|^2\cdot \EE_{ s_j}\{\indi{(s_i=s_j)}\}\} }\\
&= \sqrt{\frac{1}{d}}
\sqrt{
\EE\{\|\hat\vecx_{j-3\tau}-\hat\vecx_i\|^2\}\cdot 
\EE\{\|\hat\vecx_{j-3\tau}-\hat\vecx_j\|^2\}}\\
&\le \ct \sqrt{\frac{1}{d}}
\cdot\gamma^2\tau^2 \underbrace{\max_{k \in \setS_4^j}\EE\{\|\vecg(\hat\vecx_k,s_k)\|^2\}}_{G_4},
\end{align*}
 where the first equality follows due to $\hat\vecx_j$ being independent of $s_j$; hence the expectation with respect to $s_j$ can be applied to the indicator function. The last inequality follows from our arguments in the proof of Lemma~\ref{lemma:recursions ascd} because both mismatches $\hat\vecx_{j-3\tau}-\hat\vecx_i$ and $\hat\vecx_{j-3\tau}-\hat\vecx_j$ can be written as linear combinations of gradient steps indexed by $\setS_{4}^j$ as in \eqref{eq: shift decomposition ascd}.
Similarly the third term satisfies the inequality 
\begin{align*}
\EE\{\|\vecg(\hat\vecx_{j-3\tau},s_j)\|&\cdot \|\hat\vecx_{j-3\tau}-\hat\vecx_i\|\cdot \indi{(s_i=s_j)}\}\\
&\le 
\sqrt{
\EE\{\|\vecg(\hat\vecx_{j-3\tau},s_j)\|^2\}\cdot 
\EE\{\|\hat\vecx_{j-3\tau}-\hat\vecx_i\|^2\cdot \indi{(s_i=s_j)}\} }\\
&= \ct \sqrt{\frac{1}{d}}\cdot \gamma \tau G_4.
\end{align*}
The same bound applies for the fourth term
$\EE\{\|\vecg(\hat\vecx_{j-3\tau},s_i)\|\cdot \|\hat\vecx_{j-3\tau}-\hat\vecx_j\|\cdot \indi{(s_i \! = \! s_j)}\}$, while 
the first term can be easily bounded as 
\begin{align*}
\EE\{\|\vecg(\hat\vecx_{j-3\tau},s_j)\|\cdot \|\vecg(\hat\vecx_{j-3\tau},s_i)\|\cdot \indi{(s_i=s_j)}\}
&\le \sqrt{\frac{1}{d}}G_4.
\end{align*}
Putting all pieces together, and using the prescribed value of 
$\gamma = \frac{\theta}{6dL\kappa}$, we have that
$$
R_2 \leq \ct \sqrt{\frac{1}{d}}(\gamma\tau)\left(1 + dL\gamma\tau + (dL\gamma\tau)^2\right) G_4 \leq \ct \sqrt{\frac{1}{d}}\cdot \gamma\cdot \tau^3 \cdot G_4.
$$
The first inequality follows because we are summing over $2\tau$ terms in \eqref{eq:R2 sum expansion ascd}. To see why the second inequality is true, note that $dL\gamma \leq \frac{\theta}{6\kappa} \leq 1$ (it is always true that the condition number $\kappa \geq 1$ ). Therefore  
$
1 + dL\gamma\tau + (dL\gamma\tau)^2 \leq 1 + \tau + \tau^2 \leq 3\tau^2.
$
As in the proof of Lemma~\ref{lemma:R01 bound ascd}, we can bound $G_4$ by
\begin{align*}
G_4 &\leq \ct A \sum_{i = 0}^{\ell - 1}(3^i \cdot (i + 4)!)^2(BC)^i+ \ct (3^{\ell}\cdot (\ell + 4)!)^2(BC)^\ell G_{\ell + 4}\\
&\leq \ct (dL^2 a_j + \theta^{2\ell}dM^2).
\end{align*}
The result follows assuming $\tau = \mathcal{O}(\sqrt[6]{d})$ and $\gamma = \frac{\theta}{6dL \kappa}$.   \qquad
\end{proof}

\begin{remark}
We believe that if we use the same bounding technique that we applied for $R_2^j$ on $R_0^j$ and $R_1^j$, then we can significantly improve the restrictive bound on $\tau$.
\end{remark}

\subsection{\CM{}}
\label{sec:svrg_proofs}
\subsubsection{Bounding $R_0^j$ and $R_1^j$}
\svrgRzeroone*
\begin{proof}
Let $A = 4L^2( a_j +a_0)$, $B = 4L^2$, and $C = (\gamma\tau)^2$.
Then,  the inequalities derived above can be rewritten as
\begin{equation}
G_r \leq A+B\Delta_r \text{ and } \Delta_r \leq 3^2(r+1)^2C G_{r+1}.
\end{equation}
If we expand the formulas, we get for $R_0^j$ the following upper bound
$$
R_0^j = G_0 \leq A\sum_{i=0}^{\ell-1} (3^i\cdot i!)^2 (BC)^i+ (3^\ell\cdot \ell!)^2 (BC)^\ell G_{\ell}.
$$
We chose $\gamma = \frac{\theta}{12L\kappa}$ and assumed that $\tau \leq \kappa/\ell$, where $\kappa = \frac{L}{m}$ is the condition number and $\theta \leq 1$. We chose $\gamma$ to be proportional to the step-size of the serial SVRG, and the assumption on $\tau$ is made so that the sum in the above inequality is significantly is significantly small. Then,
$$(3^i\cdot i!)^2 (BC)^i\le \left(3 i\right)^{2i}\left(4L^2\frac{\theta^2}{4^2\cdot 3^2L^2 \kappa^2}\frac{\kappa^2}{\ell^2}\right)^i\leq \frac{\theta^2}{4^i}\left(\frac{i}{\ell}\right)^{2i}$$
and hence
$$ \sum_{i=0}^{l-1}(3^i\cdot i!)^2 (BC)^i \le \sum_{i=0}^\infty 2^{-2i} \leq 2.$$
As in the previous sections we assume a uniform upper bound $M > 0$ on the size of the gradient steps: $\max_j\EE\|\vecv(\hat\vecx_j, s_j)\|^2 = M^2$. Therefore
$$R^j_0 =G_0 \leq \mathcal{O}(1)\left(L^2(a_j +a_0)+ \theta^{2\ell} M^2\right).$$ 
After an analogous derivation one can see that
$$
R_1^j = \Delta_0 \leq \ct\left(\theta^2(a_j + a_0) + \theta^{2\ell}\frac{M^2}{L^2}\right),
$$
and thus we obtain the result. \qquad
\end{proof}

\subsubsection{Bounding $R_2^j$}
\label{sec:svrgR2}

\svrgRtwo*

\begin{proof}
From \eqref{eq: shift decomposition svrg} we can upper bound $R_2^j$ as follows. 
\begin{equation}
\label{eq:R2 sum expansion}
R_2^j = \EE \langle\hat\vecx_j-\vecx_j,\vecv(\hat\vecx_j,s_j)\rangle 
\le 
\gamma\cdot \sum_{i=j-\tau, i \ne j}^{j+\tau} \EE \|\vecv(\hat{\vecx}_i, s_i)\|\cdot \|\vecv(\hat\vecx_j,s_j)\| \cdot \indi{(s_i \cap s_j \neq \emptyset )}.
\end{equation}
The random variable  $\indi{(s_i \cap s_j \neq \emptyset )}$ encodes the sparsity of the gradient steps. As in the proof of Lemma~\ref{lemma:R2 ascd bound}, we replace $\hat\vecx_i$ and $\hat\vecx_j$ in the above sum by $\hat\vecx_{j - 3\tau}$. When $j < 3\tau$ we define $\hat\vecx_{j - 3\tau} = \vecx_0$. Since $f_{e_i}$ are $L$-smooth, we have
\begin{align*}
\|\vecv(\hat\vecx_j,s_j)\| &\le \|\vecv(\hat\vecx_{j-3\tau},s_j)\|+L\|\hat\vecx_{j-3\tau}-\hat\vecx_j\|\\
\|\vecv(\hat\vecx_i,s_i)\| & \le \|\vecv(\hat\vecx_{j-3\tau},s_i)\|+L\|\hat\vecx_{j-3\tau}-\hat\vecx_i\|.
\end{align*}
Then, the expectation of a term $\|\vecv(\hat{\vecx}_i, s_i)\|\cdot \|\vecv(\hat\vecx_j,s_j)\| \cdot \indi{(s_i\cap s_j)}$ in the sum \eqref{eq:R2 sum expansion ascd} is upper bounded by
\begin{align*}
\EE\biggl\{
\biggl(&\|\vecv(\hat\vecx_{j-3\tau},s_i)\|\; \|\vecv(\hat\vecx_{j-3\tau},s_j)\| + L^2  \|\hat\vecx_{j-3\tau}-\hat\vecx_i\|\; \|\hat\vecx_{j-3\tau}-\hat\vecx_j\|\\
+&
L  \|\vecv(\hat\vecx_{j-3\tau},s_j)\|\; \|\hat\vecx_{j-3\tau}-\hat\vecx_i\| + L  \|\vecv(\hat\vecx_{j-3\tau},s_i)\|\; \|\hat\vecx_{j-3\tau}-\hat\vecx_j\| \biggr)\cdot \indi{(s_i\cap s_j)}
\biggr\}.
\end{align*}
Then, since $\EE \indi{(s_i \cap s_j \neq \emptyset )} \leq \frac{2\cdeg}{n}$ (recall that $\cdeg$ is the average conflict degrees), $R_2^j$ can be shown to satisfy the inequality
$$
R_2^j \leq \ct\sqrt{\frac{\cdeg}{n}}\gamma\tau^3\left(L^2(a_j+ a_0) + \theta^{2\ell}M^2\right)
$$
as in the proof of Lemma~\ref{lemma:R2 ascd bound}.
The conclusion follows because $\tau = \mathcal{O}\left(\sqrt[6]{\frac{n}{\cdeg}}\right)$ and $\gamma = \frac{\theta}{12L\kappa} = \frac{m\theta}{12L^2}$. \qquad
\end{proof}

\begin{remark}
Similar to ASCD, by using the same bounding technique of $R_2^j$ on $R_0^j$ and $R_1^j$, we should significantly improve the restrictive bound on $\tau$ in the convergence result of \CM{}.
\end{remark}